\documentclass{article}

\usepackage[nonatbib, preprint]{neurips_2026}
\usepackage[numbers]{natbib}

\usepackage[utf8]{inputenc} % allow utf-8 input
\usepackage[T1]{fontenc}    % use 8-bit T1 fonts
\usepackage{hyperref}       % hyperlinks
\usepackage{url}            % simple URL typesetting
\usepackage{booktabs}       % professional-quality tables
\usepackage{amsfonts}       % blackboard math symbols
\usepackage{nicefrac}       % compact symbols for 1/2, etc.
\usepackage{microtype}      % microtypography
\usepackage{xcolor}         % colors
\usepackage{graphicx}
\usepackage{subcaption}
\usepackage{booktabs, multirow, adjustbox}
\usepackage{bbold}

\usepackage{amsmath}
\usepackage{amsthm}
\usepackage{amssymb}
\usepackage{mathtools}
\usepackage{stmaryrd}
\usepackage{bbm}
\usepackage{thmtools, thm-restate}
\newtheorem{definition}{Definition}

\newtheorem{theorem}{Theorem}
\newtheorem{proposition}{Proposition}
\newtheorem{remark}{Remark}

\title{Generalized Conformal Predictive Systems Under Distributional Shifts}

\author{%
  Jef Jonkers \\
  IDLab \\
  Department of Electronics \\
and Information Systems \\
  Ghent University \\
  Ghent, Belgium \\
  \texttt{jef.jonkers@ugent.be} \\
  \And
  Johanna Ziegel \\
  Seminar for Statistics \\ 
  ETH Zurich \\ 
  Zurich, Switzerland \\
  \texttt{ziegelj@ethz.ch}
}

\begin{document}

\maketitle

\begin{abstract}
    Conformal predictive systems (CPS) output calibrated bands of CDFs under exchangeability. We extend generalized CPS to non-exchangeable settings by encoding distributional shifts through observation-specific permutation weights. This yields shift-aware predictive systems that remain valid whenever the test point is, conditionally on the unordered sample, a weighted draw from the observed atoms. Since such weights are typically estimated, we introduce weight-uncertainty boxes and construct robust CPS envelopes with finite-sample or asymptotic confidence guarantees. We derive efficient computation for conformity-measure CPS, conformal binning, and conformal isotonic distributional regression. Experiments under covariate shift and feedback-driven biomolecular design show calibrated predictive bands that widen under stronger shifts and tighten as sample size increases.
\end{abstract}

\section{Introduction}
Reliable uncertainty quantification is central to deploying machine learning systems in decision-making settings. Beyond accurate point predictions, practitioners increasingly require \emph{predictive distributions} that are statistically compatible with realized outcomes, so that downstream quantities (e.g., tail risks, threshold events, and prediction intervals at multiple levels) can be trusted \citep{gneiting_probabilistic_2007}. A classical way to formalize this requirement is \emph{calibration}: predicted probabilities should match empirical frequencies, prediction intervals should attain their target coverage, and predictive distributions should induce calibrated intervals across levels.
In distributional regression, the goal is therefore to produce predictive distributions that are as informative as possible while remaining calibrated.

A major obstacle is that forecasting is inherently out-of-sample: we predict outcomes that have not yet been observed. Guaranteeing out-of-sample calibration for a \emph{single} predictive distribution is generally out of reach without strong assumptions. Conformal methods sidestep this impossibility by outputting a \emph{set} of probabilistic forecasts, i.e., a credal set, that is guaranteed to contain at least one calibrated forecast under weak symmetry assumptions. For real-valued outcomes, these credal sets take the form of \emph{predictive systems}: sets of CDFs represented by lower and upper CDF envelopes \citep{vovk_nonparametric_2019,allen_-sample_2025}. Conformal predictive systems (CPS) provide such envelopes with probabilistic calibration guarantees \citep{vovk_nonparametric_2019}, and generalized conformal predictive systems (GCPS) extend this idea to broader calibration notions by leveraging in-sample calibrated distributional regression procedures \citep{allen_-sample_2025}.

A distinctive feature of predictive systems is that their width is informative: it quantifies the epistemic uncertainty in what calibration can be guaranteed from the available data \citep{allen_-sample_2025}. Recent work by \citet{allen_-sample_2025} explores the conceptual framework underlying GCPS. Rather than viewing them as tied to conformity scores, one can see them as \emph{calibration wrappers}: if we have a procedure that is calibrated \emph{in-sample}, then a conformal-style construction can turn it into a predictive system that contains an out-of-sample calibrated predictive distribution under exchangeability \citep{allen_-sample_2025}. A closely related perspective appears for point forecasting \citep{van_der_laan_generalized_2025}. This viewpoint motivates a simple question: \emph{which in-sample calibrated forecasting procedures can we wrap, and how far beyond exchangeability can we push the resulting guarantees?}
For distributional regression, prominent in-sample calibrated building blocks include binning-based methods and isotonic distributional regression (IDR) \citep{henzi_isotonic_2021}.

In this paper, we focus on the next bottleneck: \emph{distributional shift}. The exchangeability assumption underpinning the classical out-of-sample guarantees is routinely violated by covariate shift, label shift, temporal drift, and adaptive or feedback-driven data collection. Prior work has already begun extending CPS beyond exchangeability in specific regimes: for instance, \citet{jonkers_conformal_2024, jonkers_conformal_2025} develop weighted CPS for \emph{known covariate shift} using likelihood-ratio weights, focusing on conformity-measure CPS and probabilistic calibration. Our key idea is to retain the predictive-system viewpoint while relaxing exchangeability through a lightweight interface based on \emph{permutation weights}. Conditioning on the unordered content of the observed sample, non-exchangeability can be represented as a weighted lottery over which observed atom becomes the test point \citep{tibshirani_conformal_2019, prinster_conformal_2024}. These weights allow us to extend GCPS to non-exchangeable sequences and to construct predictive envelopes that remain calibrated under arbitrary distributional shifts.

In practice, the shift weights are rarely known and must be estimated. To account for this additional epistemic uncertainty, we introduce \emph{weight-uncertainty boxes}: coordinatewise multiplicative bounds around estimated weights. Optimizing the predictive system over this box yields robust CDF envelopes with confidence guarantees that the envelope contains a calibrated predictive distribution, while the envelope thickness reflects both finite-sample uncertainty and weight misspecification.

\paragraph{Contributions.} To push GCPS beyond exchangeability while retaining calibrated predictive distributions, we make the following contributions: 
(i) We extend GCPS to non-exchangeable sequences by encoding departures from exchangeability via permutation weights and constructing shift-aware predictive envelopes.
(ii) We develop robust predictive systems under \emph{weight uncertainty}, based on weight-uncertainty boxes, with finite-sample or asymptotic confidence guarantees that the envelope contains a calibrated predictive CDF.
(iii) We provide efficient evaluation rules and closed forms for important instances (including binning- and IDR-based procedures), and demonstrate empirically that shift-aware weighting improves predictive performance while preserving calibration under feedback covariate shift \citep{fannjiang_conformal_2022}.

\section{Problem setup}
\label{sec:setup}
\paragraph{Data and goal.}
We observe covariate-response pairs $Z_i=(X_i,Y_i)\in\mathcal Z:=\mathcal X\times\mathbb R$ for $i=1,\dots,n$ and a new covariate $X_{n+1}$, with an associated (unobserved) response $Y_{n+1}$. Our goal is to output, given a test covariate $X_{n+1}$, a GCPS that is guaranteed to contain at least one calibrated predictive distribution for $Y_{n+1}$, even when the joint sequence $(Z_1,\dots,Z_{n+1})$ is \emph{not} exchangeable.

\begin{definition}[Predictive system]
Let $\mathcal M_1$ be the set of probability measures on $\mathcal Z$ with finite support. A \emph{distributional regression procedure} is a measurable map $G:\mathcal M_1\times\mathcal X\times\mathbb R\to[0,1]$, $(\mu,x,y)\mapsto G[\mu,x](y)$, such that for each $(\mu,x)$ the function $y\mapsto G[\mu,x](y)$ is a CDF.

Given a finite (nonnegative) measure $\mu$ on $\mathcal Z$, a covariate $x$, and an insertion weight $c>0$, define the predictive-system envelopes for all $y \in \mathbb{R}$ by
\begin{align}
    \Pi_\ell[\mu,x;c](y) &:= \inf_{y'\in\mathbb R} G\!\left[P_{\mu+c\,\delta_{(x,y')}},\,x\right](y), 
    \quad 
    \Pi_u[\mu,x;c](y) := \sup_{y'\in\mathbb R} G\!\left[P_{\mu+c\,\delta_{(x,y')}},\,x\right](y), \label{eq:PS_up_main}
\end{align}
where the infimum and supremum are taken over candidate insertion values $y'\in\mathbb R$ for the test response (\emph{not} over the calibration sample). Let $\Pi[\mu,x;c]$ denote the set of CDFs whose graphs lie between $\Pi_\ell[\mu,x;c]$ and $\Pi_u[\mu,x;c]$. Here, $P_\nu := \nu/\nu(\mathcal Z)$ denotes the normalized version of any finite nonnegative measure $\nu$ with $\nu(\mathcal Z) > 0$. 
\end{definition}

\paragraph{Non-exchangeability through permutation weights.}
Instead of assuming exchangeability, we model ordering effects through \emph{permutation weights}. Informally, we condition on the unordered multiset of the $(n{+}1)$ realized atoms and allow the identity of the test atom among them to be non-uniform. Formally, for a sequence $z_{1:n+1}\in\mathcal Z^{n+1}$, let $w_i(z_{1:n+1})\in[0,1]$, $i=1,\dots,n+1$, be the probability that the test point $Z_{n+1}$ equals the $i$th atom conditional on the unordered set of observed atoms, that is,
\begin{equation}
  \label{eq:weight-measure}
  \mathbb P\!\left(Z_{n+1}=z_i \,\middle|\, \{Z_1,\dots,Z_{n+1}\}=\{z_1,\dots,z_{n+1}\}\right) = w_i(z_{1:n+1}).
\end{equation}
When the data are exchangeable, $w_1=\cdots=w_{n+1}=1/(n+1)$. While the weights defined at \eqref{eq:weight-measure} are normalized by definition, in all the following arguments, one can also work with non-normalized weights without any changes. For ease of notation, we assume that all components of $Z_{1:n+1}$ are distinct almost surely. All our results also hold if repeated values occur with positive probability, but the notation would be slightly more involved.

\paragraph{Weighted empirical law at test time.}
Given a candidate value $y' \in \mathbb R$ and weights $w^\star_i(y') = w_i(Z_{1:n},(X_{n+1},y'))$, define the weighted observed measure and the (hypothetical) full weighted measure
\begin{equation}
    \label{eq:mu_obs_tilde}
    \mu_{\text{obs}}^{y'} := \sum_{i=1}^n w^\star_i(y')\,\delta_{Z_i},
    \qquad
    \widetilde\mu := \sum_{i=1}^{n+1} w^\star_i(Y_{n+1})\,\delta_{Z_i}.
\end{equation}
Since $\sum_{i=1}^{n+1} w^\star_i(y') = 1$, $\widetilde\mu$ is already a probability measure and $P_{\widetilde\mu} = \widetilde\mu$. In the shift-aware construction (Section~\ref{sec:wgcps}), we set the insertion weight in \eqref{eq:PS_up_main} to the realized test weight $c=w^\star_{n+1}(y')$ so that $\mu_{\text{obs}}^{Y_{n+1}}+w^\star_{n+1}(Y_{n+1})\delta_{(X_{n+1},Y_{n+1})}=\widetilde\mu$. This alignment is what enables calibrated predictive distributions under non-exchangeability.

\paragraph{Known vs.\ uncertain weights.}
In some settings, the weights are known or computable from the data-generating mechanism (e.g., covariate shift with known likelihood ratios), see Appendix \ref{app:permweights}. More commonly, we only have estimated (possibly misspecified) weights $\hat w_i=\hat w_i(Z_{1:n+1})\ge 0$ and must account for additional epistemic uncertainty with regard to the true weights $w^\star=(w_1^\star,\dots,w_{n+1}^\star)$.
We model this uncertainty through a \emph{weight-uncertainty box} (coordinatewise multiplicative bounds) $L_i\,\hat w_i \;\le\; w^\star_i \;\le\; U_i\,\hat w_i$, for $i=1,\dots,n+1$, where $0\le L_i\le U_i<\infty$ may depend on the sample and a user-chosen level $\alpha$.
This defines
\begin{equation}
    \label{eq:W_WU}
    \mathcal{W}^{\text{WU}}_\alpha :=\Big\{
        w\in\mathbb R_+^{n+1}: \ L_i\hat w_i \le w_i \le U_i\hat w_i\ \text{for all } i\Big\},
\end{equation}
which we assume to be a $(1-\alpha)$ confidence region for $w^\star$, i.e.
\begin{equation}
    \label{eq:WU_coverage}
    \mathbb P\!\left(w^\star\in\mathcal W^{\text{WU}}_\alpha\right)\ge 1-\alpha,
\end{equation}
either as an exact (or conservative) finite-sample guarantee for each fixed $n$, or as an asymptotic guarantee in the sense that $\liminf_{n\to\infty}\mathbb P(w^\star\in\mathcal W^{\text{WU}}_\alpha)\ge 1-\alpha$.

In Section~\ref{sec:weight-robust} we construct the predictive-system envelopes over $\mathcal W^{\text{WU}}_\alpha$ to obtain robust CDF bands with confidence guarantees.

\section{Generalized conformal predictive systems under distributional shifts}
\label{sec:wgcps}

This section shows how the GCPS \citep{allen_-sample_2025} extends beyond exchangeability. The key observation is that, conditional on the unordered multiset of atoms, the test point $Z_{n+1}$ is distributed as a \emph{weighted draw} from the $n+1$ observed atoms with probabilities $w_i^*(Y_{n+1})$ from~\eqref{eq:weight-measure}. Consequently, if we build the predictive system using the weighted empirical law $P_{\widetilde\mu}$ in~\eqref{eq:mu_obs_tilde} and insert the test atom with its realized weight $w^\star_{n+1}(Y_{n+1})$, then the predictive system contains a calibrated predictive distribution for $Y_{n+1}$ under the true (possibly non-exchangeable) law. We use the definition of \emph{in-sample isotonic calibration} and \emph{in-sample auto-calibration} as given in \citet[Definition 2.2]{allen_-sample_2025}.

\begin{theorem}[GCPS under known distributional shifts]
    \label{thm:non-exchangeable-known-weights}
    Suppose $Z_{1:n+1}=(Z_1,\dots,Z_{n+1})$ with $Z_i=(X_i,Y_i)\in\mathcal Z$ are drawn from an arbitrary joint law. Let $\mu_{\text{obs}}^{y'}$ and $\widetilde\mu$ be the weighted measures from~\eqref{eq:mu_obs_tilde}.
    
    Let $G$ be a distributional regression procedure and construct the predictive system $\Pi[X_{n+1}]$ via the bounds
    \begin{align*}
    \Pi_\ell[X_{n+1}](y) &:= \inf_{y'\in\mathbb R} G\!\left[P_{\mu_{\text{obs}}^{y'}+w_{n+1}^\star(y')\,\delta_{(X_{n+1},y')}},\,X_{n+1}\right](y), \\
    \Pi_u[X_{n+1}](y) &:= \sup_{y'\in\mathbb R} G\!\left[P_{\mu_{\text{obs}}^{y'}+w_{n+1}^\star(y')\,\delta_{(X_{n+1},y')}},\,X_{n+1}\right](y).
    \end{align*} 
    Assume $\mathcal Z$ is standard Borel and that $(\mu,x)\mapsto G[\mu,x](\cdot)$ is measurable. %Let $H$ be any summary function taking values in a measurable space $(\mathcal H,\mathcal B_{\mathcal H})$, and assume that $(\mu,x)\mapsto H[\mu,x]$ is measurable (with respect to the product $\sigma$-field on $\mathcal M_1\times\mathcal X$).
    
    Define the (random) predictive CDF $F^\star(\cdot) := G[P_{\widetilde\mu},\,X_{n+1}](\cdot)$. Then $F^\star\in \Pi[X_{n+1}]$ almost surely. Moreover:

    \begin{enumerate}
        \item[(i)] \emph{Probabilistic calibration.} If $G$ satisfies in-sample probabilistic calibration (for every finite-support law), then
        \begin{equation*}
            \mathbb{P}\big(F^\star(Y_{n+1})\le u\big)\le u \le \mathbb{P}\big(F^\star(Y_{n+1}-)< u\big) \qquad \text{ for all } u\in(0,1).
        \end{equation*}
        \item[(ii)] \emph{Isotonic calibration.} If $G$ satisfies in-sample isotonic calibration, then
        \begin{equation*}
            \mathbb{P}\big(Y_{n+1}>y \mid \mathcal A(F^\star)\big) \;=\; 1-F^\star(y) \qquad \text{ for all } y\in\mathbb R.
        \end{equation*}
        Here, $\mathcal A(F^\star)$ is the $\sigma$-lattice generated by $F^\star$ with respect to usual stochastic order \citep{arnold_isotonic_2025}.
        \item[(iii)] \emph{Auto-calibration.} If $G$ satisfies in-sample auto-calibration, then
        \begin{equation*}
            \mathbb{P}\big(Y_{n+1}>y \mid F^\star \big) \;=\; 1-F^\star(y)  \qquad \text{ for all } y\in\mathbb R.
        \end{equation*}
    \end{enumerate}
\end{theorem}

The prime example of a distributional regression procedure $G$ that satisfies in-sample probabilistic calibration is a classical CPS, for isotonic calibration, it is conformal IDR, and for auto-calibration, it is any binning based method. Details for these three regression procedures are given in Section \ref{sec:comp}.

% \begin{proof}
%     See Appendix ~\ref{sec:wcps-proof}.
% \end{proof}

From now on, we assume that $w_i^\star = w_i^\star(y')$ (and hence also $\mu_{\text{obs}}$) does not depend on $y'$, that is, we are in a covariate shift scenario, see Appendix \ref{app:permweights}. In this case,  $\Pi[X_{n+1}] = \Pi[\mu_{\text{obs}},X_{n+1};w_{n+1}^\star]$, where the right hand side is constructed via~\eqref{eq:PS_up_main}.

\begin{remark}[Computability and avoiding a sweep over $y'$]
    The defining envelopes~\eqref{eq:PS_up_main} take an infimum/supremum over $y'\in\mathbb R$, which naively suggests evaluating $G[P_{\mu_{\text{obs}}+w^\star_{n+1}\delta_{(X_{n+1},y')}},X_{n+1}]$ for many candidate values of $y'$ across the response domain. In practice, this is unnecessary: the same monotonicity and splitting tricks as in \citet{allen_-sample_2025} apply verbatim in our shift-aware setting, since the only change is that $\mu$ is replaced by a weighted measure and the test atom is inserted with weight $w^\star_{n+1}$.

    Concretely, if $G$ is implemented in a \emph{split} manner (parameters of $G$ estimated on an independent estimation set, then held fixed on the calibration set), then evaluating the predictive system for a single test covariate reduces to computations on the calibration set only. Moreover, for many useful choices of $G$ the mapping $y' \mapsto G\!\left[P_{\mu_{\text{obs}}+w^\star_{n+1}\delta_{(X_{n+1},y')}},\,X_{n+1}\right](y)$ is monotone for each fixed $y$, so the inner infimum and supremum are attained at extreme insertions. \emph{Two evaluations of the underlying procedure suffice} in the cases we treat: (i)~conformal binning (CBIN) admits closed-form envelopes obtained by setting $y'$ above and below the bin breakpoint of interest; (ii)~monotone conformity-score CPS (CMCPS) admits the same closed form after replacing outcomes by scores; and (iii)~conformal IDR (CIDR) is obtained by running IDR twice, with the test label set to an extreme low and extreme high value \citep[cf.][Remark~2.9]{allen_-sample_2025}, with weights carried throughout.
\end{remark}

\begin{remark}[Venn-Abers under distributional shift via weighted conformal IDR]
    Venn-Abers predictors \citep{vovk_venn-abers_2014} can be viewed as a special case of conformal IDR applied to binary outcomes (or, equivalently, to threshold events), and thus fall within the broader class of \emph{conformal IDR} methods that combine IDR with GCPS \citep{allen_-sample_2025}. Consequently, our shift-aware GCPS construction, when instantiated with conformal IDR, yields Venn-Abers type calibrated probability forecasts \emph{beyond exchangeability}: the resulting credal set contains a forecast that is calibrated for the out-of-sample test outcome under distributional shifts encoded by the permutation weights $w_i$.
    In addition, the weight-robust construction in Section~\ref{sec:weight-robust} provides a robust analogue when the shift encoded weights are uncertain. To the best of our knowledge, this establishes the first out-of-sample calibration guarantee for Venn-Abers style predictors under distributional shifts (and, correspondingly, under weight uncertainty in the shift model) within the unified predictive-systems framework.
\end{remark}

\section{Weight-uncertain robust conformal predictive systems}
\label{sec:weight-robust}
Section~\ref{sec:wgcps} establishes that, under \emph{known} permutation weights $w^\star_i$, the GCPS $\Pi[\mu_{\text{obs}},X_{n+1};w^\star_{n+1}]$ contains a calibrated predictive CDF for $Y_{n+1}$. In practice, the weights are typically estimated (e.g., via importance weights/density ratios) and may be misspecified. We therefore propagate \emph{weight uncertainty} through the predictive system by defining the GCPS envelopes over the weight-uncertainty box $\mathcal W^{\text{WU}}_\alpha$ from~\eqref{eq:W_WU} which is assumed to satisfy \eqref{eq:WU_coverage}.

For each candidate weight vector $w\in\mathcal W^{\text{WU}}_\alpha$, define the corresponding observed measure $\mu_{\text{obs}}(w) \;:=\; \sum_{i=1}^n w_i\,\delta_{Z_i}$, and build the predictive-system envelopes at $x=X_{n+1}$ with insertion weight $c=w_{n+1}$ using \eqref{eq:PS_up_main}. We then define the \emph{weight-robust} CDF envelopes
\begin{equation}
    \label{eq:WU-env-lower}
    \Pi^{\text{WU},\alpha}_\ell(y) :=  \inf_{w\in\mathcal W^{\text{WU}}_\alpha}\ \Pi_\ell[\mu_{\text{obs}}(w),\,X_{n+1};\,w_{n+1}](y),
\end{equation}

\begin{equation}
    \label{eq:WU-env-upper}
    \Pi^{\text{WU},\alpha}_u(y) := \sup_{w\in\mathcal W^{\text{WU}}_\alpha}\ \Pi_u[\mu_{\text{obs}}(w),\,X_{n+1};\,w_{n+1}](y).
\end{equation}
Let $\Pi^{\text{WU},\alpha}$ denote the set of CDFs that lie between $\Pi^{\text{WU},\alpha}_\ell$ and $\Pi^{\text{WU},\alpha}_u$.

\begin{theorem}[Weight-robust GCPS with confidence]
    \label{thm:WU-robust-alpha}
    Assume $G$ satisfies an in-sample calibration property for every finite-support law. Let $w^\star$ denote the (unknown) true permutation weights and define $\Pi^{\text{WU},\alpha}$ by~\eqref{eq:WU-env-lower}--\eqref{eq:WU-env-upper}. By Theorem \ref{thm:non-exchangeable-known-weights}, $F^\star(\cdot) := G[P_{\widetilde\mu},\,X_{n+1}](\cdot)$ is calibrated.
    
    \begin{enumerate}
    \item[(a)] \textnormal{(Finite-sample confidence guarantee).}
    If $\mathcal W^{\text{WU}}_\alpha$ satisfies the coverage condition~\eqref{eq:WU_coverage}, then     $\mathbb{P}(F^\star \in \Pi^{\text{WU},\alpha})\;\ge\; 1-\alpha$.
    \item[(b)] \textnormal{(Asymptotic analogue).}
    If $\mathcal W^{\text{WU}}_\alpha$ satisfies an asymptotic version of~\eqref{eq:WU_coverage}, then $\liminf_{n\to\infty}\mathbb{P}(F^\star \in \Pi^{\text{WU},\alpha})\;\ge\; 1-\alpha$.
    \end{enumerate}
\end{theorem}

% \begin{proof}
%     See Appendix \ref{app:WU-proof}.
% \end{proof}

\begin{remark}[Deterministic weight-uncertainty boxes]
    Taking $\alpha=0$ yields a deterministic uncertainty set $\mathcal W^{\text{WU}}:=\mathcal W^{\text{WU}}_{\alpha=0}$. In this case, the robust envelopes are valid whenever the model restricts $w^\star\in\mathcal W^{\text{WU}}$ almost surely.
\end{remark}

\subsection{Computability}\label{sec:comp}
We now spell out how to compute the weight-robust envelopes $\Pi^{\text{WU},\alpha}_\ell,\Pi^{\text{WU},\alpha}_u$ from~\eqref{eq:WU-env-lower}--\eqref{eq:WU-env-upper} for three practically important choices of $G$; full derivations and additional technical details are given in Appendix~\ref{app:computation}. Throughout, fix $x:=X_{n+1}$ and define the per-atom lower/upper weight bounds $l_i := L_i\hat w_i, u_i := U_i\hat w_i, i=1,\dots,n+1$.

\paragraph{Conformal binning (CBIN).}
Suppose that a partition of $\{1,\dots,n+1\}$ has been determined either on an independent training set (split setting) or using covariates $x_1,\dots,x_n$, only. 
Let $I(x)\subseteq\{1,\dots,n\}$ denote the indices in the same partition element as $n+1$ (e.g.\ the bin of $x$).
Define $G$ as the weighted empirical CDF within the bin $I(x)$ augmented by the test atom with weight $w_{n+1}$:
\begin{equation}
    G\!\left[P_{\mu_{\text{obs}}(w)+w_{n+1}\delta_{(x,y')}},\,x\right](y) = \frac{\sum_{i\in I(x)} w_i\,\mathbb 1\{Y_i\le y\} + w_{n+1}\,\mathbb 1\{y'\le y\}}
        {\sum_{i\in I(x)} w_i + w_{n+1}}.
\end{equation}
This procedure is in-sample auto-calibrated \citep[Section 2.4]{allen_-sample_2025}.
The inner $\inf_{y'}$ / $\sup_{y'}$ in \eqref{eq:PS_up_main} is explicit: to maximize the CDF at threshold $y$, choose $y'\le y$ (so the test atom contributes to the numerator); to minimize it, choose $y'>y$. Optimizing over the weight-uncertainty box then reduces to pushing atoms with $Y_i\le y$ to their upper bounds and atoms with $Y_i>y$ to their lower bounds for the upper envelope, and vice versa for the lower envelope, while always taking $w_{n+1}$ at its upper bound to widen the credal set. Concretely, for any $y\in\mathbb R$,

\begin{equation}
    \label{eq:WU-cbin-upper}
    \Pi^{\text{WU},\alpha}_u(y)
    = \frac{ u_{n+1}+\sum_{i\in I(x):\,Y_i\le y} u_i}
    { u_{n+1}+\sum_{i\in I(x):\,Y_i\le y} u_i+\sum_{i\in I(x):\,Y_i>y} l_i},
\end{equation}
\begin{equation}
    \label{eq:WU-cbin-lower}
    \Pi^{\text{WU},\alpha}_\ell(y) = 
    \frac{\sum_{i\in I(x):\,Y_i\le y} l_i}
        { u_{n+1}+\sum_{i\in I(x):\,Y_i\le y} l_i+\sum_{i\in I(x):\,Y_i>y} u_i}.
\end{equation}
Hence, the envelopes are piecewise-constant with breakpoints at $\{Y_i:i\in I(x)\}$. Computationally, sort $Y_i$ on $I(x)$ once ($O(k\log k)$), where $k:=|I(x)|$, compute prefix sums of $ l_i$ and $ u_i$, and evaluate \eqref{eq:WU-cbin-upper}--\eqref{eq:WU-cbin-lower} for all breakpoints in $O(k)$.

\paragraph{Conformity-measure CPS (CMCPS).}
If $G$ is a conformity-measure CPS with a \emph{fixed} monotone score $S$ (e.g.\ split conformal), it is in-sample probabilistically calibrated \citep[Section 2.3]{allen_-sample_2025} and the same computation applies after replacing outcomes $Y_i$ by scores $s_i:=S(Z_i)$ and the threshold $y$ by $s_y:=S((x,y))$. That is, for any candidate $y$ we evaluate the envelopes in \eqref{eq:WU-cbin-upper}--\eqref{eq:WU-cbin-lower} with $Y_i$ replaced by $s_i$ and the event $\{Y_i\le y\}$ replaced by $\{s_i\le s_y\}$. Thus, once the scores are computed, evaluation is again dominated by sorting ($O(k\log k)$) plus linear-time sweeps.

\paragraph{Conformal isotonic distributional regression (CIDR).}
For conformal IDR, $G$ is defined through antitonic regression of the indicators $\mathbb 1\{Y_i\le y\}$ along a (total) order on the covariates used by IDR augmented with the test covariate at position $i^\star$ in that order. The order can be estimated on the training sample (split setting). Conformal IDR is in-sample isotonically calibrated \citep[Section 2.5]{allen_-sample_2025}. For each fixed threshold $y$, define ordered indicators $\mathbb 1\{Y'_i\le y\}$ where $(X'_i,Y'_i)$ denotes the ordered sequence used by IDR for $i=1,\dots,n+1$. Given a feasible weight vector $w$ (in that same order), the IDR predictive CDF at the test position can be written as the standard antitonic functional
\begin{equation*}
    G[P_{\mu_{\text{obs}}(w)+w_{i^\star}\delta_{(X_{i^\star}',y')}},X_{i^\star}'](y)
    := \min_{s\le i^\star}\ \max_{t\ge i^\star}
    \frac{\sum_{r=s}^t 1\{Y'_r\le y\}\,w_r}{\sum_{r=s}^t w_r},
\end{equation*}
ignoring intervals of zero total weight and implicitly setting $Y'_{i^\star} = y'$. Note that $X_{i^\star}' = X_{n+1}$. The weight-robust IDR envelopes optimize this quantity over $w\in\mathcal W^{\text{WU}}_\alpha$. A key simplification is that for each fixed $y$, the extrema are attained at \emph{box-extreme} configurations: to maximize, set $w_r= u_r$ when $1\{Y'_r\le y\}=1$ and $w_r= l_r$ when $1\{Y'_r\le y\}=0$; to minimize, swap $ u$ and $l$ (the formal statement appears in Appendix~\ref{app:wurobust-idr}). Therefore, for each $y$ we compute $\Pi^{\text{WU},\alpha}_u(y)$ and $\Pi^{\text{WU},\alpha}_\ell(y)$ by running the (IDR) antitonic regression routine twice (once at each extreme weight vector), i.e.\ two PAVA evaluations. Each PAVA run is linear in the ordered sample size ($O(n)$), so evaluating the envelopes on a grid of $m$ thresholds costs $O(mn)$ after the initial sorting/order construction; in practice, $m$ is taken as the number of distinct outcomes in the local set or a moderate evaluation grid.

% \paragraph{Summary.}
% For CBIN and conformity-score CPS, robust envelopes admit closed-form expressions with $O(k\log k)$ preprocessing and $O(k)$ evaluation at all breakpoints. For CIDR, robust envelopes reduce to two antitonic regressions per threshold, yielding $O(n)$ per threshold after the ordering is fixed.

\section{Quantifying epistemic uncertainty with conformal bands}
\label{sec:epistemic}

\begin{figure*}[!t]
    \centering
    \includegraphics[width=\textwidth]{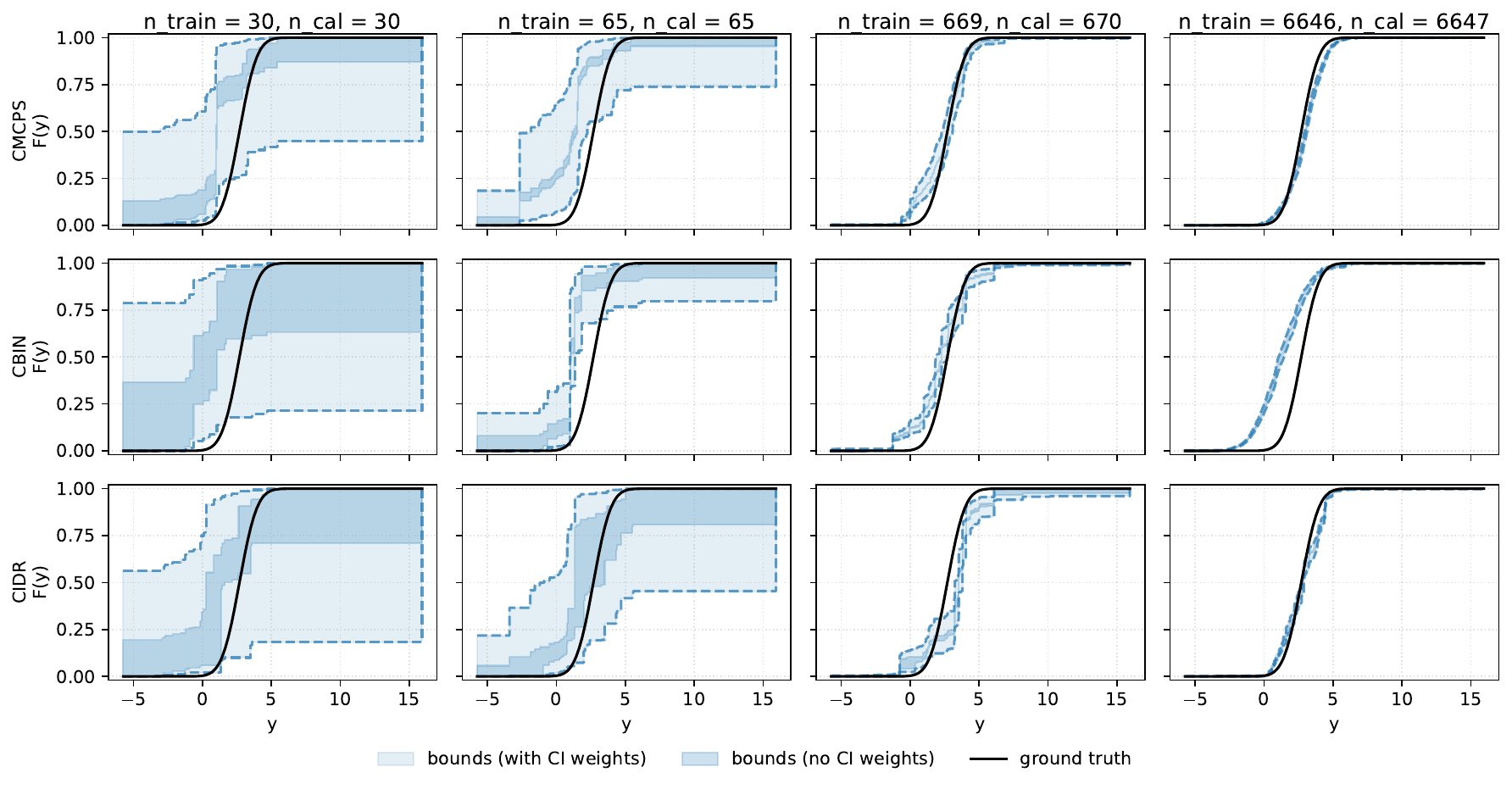}
    \caption{Effect of sample size on conformal CDF band thickness under estimated covariate shift. Using the synthetic setup described in Appendix~\ref{app:experiments}, we plot weighted conformal predictive CDF bands for a fixed target-domain test point and increasing amounts of source-domain training and calibration data (reported in the column headers $(n_{\text{train}}, n_{\text{cal}})$). Results are shown for CMCPS, CBIN, and CIDR. The lighter region uses confidence-interval-based odds weights, while the darker region uses point-estimated odds weights. As the sample size increases, the bands tighten, indicating reduced epistemic uncertainty. The black curve denotes the oracle conditional CDF.}
    \label{fig:cdf-band-example}
\end{figure*}

A GCPS outputs a \emph{conformal CDF band} $\Pi^{\text{WU},\alpha}(X_{n+1})$ that is guaranteed to contain at least one calibrated predictive CDF (Section~\ref{sec:wgcps}). The band \emph{thickness}
\begin{equation}
    \label{eq:band_thickness}
    \Delta(y\mid X_{n+1}) := \Pi^{\text{WU},\alpha}_{u}(y, X_{n+1}) - \Pi^{\text{WU},\alpha}_\ell(y, X_{n+1})
\end{equation}
provides an explicit measure of \emph{epistemic} uncertainty: a thin band indicates that calibration essentially pins down a unique forecast, while a thick band indicates ambiguity about which predictive distribution is justified by the available data and assumptions \citep{allen_-sample_2025}. In our shift-aware setting, $\Delta$ reflects both finite-sample uncertainty in learning $Y_{n+1}\mid X_{n+1}$ and additional uncertainty from estimating the shift weights, which we propagate via the weight-uncertainty box $\mathcal W^{\text{WU}}_\alpha$ (Section~\ref{sec:weight-robust}).

Figure~\ref{fig:cdf-band-example} illustrates this effect under estimated covariate shift. For a fixed test point and three instantiations of $G$ (CMCPS, CBIN, CIDR), the bands tighten as $(n_{\text{train}},n_{\text{cal}})$ increase, indicating reduced epistemic uncertainty from both model fitting and calibration. The lighter CI-weight bands are consistently wider than the darker point-weight bands, isolating the contribution of shift-estimation uncertainty, which also contracts with sample size. Complementarily, Figure~\ref{fig:cdf-band-weight-example} fixes the sample size and increases the covariate-shift weights (moving the test point further away from the training distribution); the resulting widening shows how epistemic uncertainty grows when test instances are weakly supported by the observed data. The simulation setup is documented in Appendix \ref{app:experiments}.

For deployment, $\Delta(\cdot\mid X_{n+1})$ can be summarized into a small number of levels (e.g.\ low/medium/high) as an uncertainty indicator, while still reporting a single representative forecast (e.g.\ the CRPS-optimal crisp CDF) for routine use. In higher-stakes settings, it can be preferable to report band-implied bounds on decision-relevant quantities such as tail risks or threshold probabilities $\mathbb P(Y_{n+1}\ge \tau\mid X_{n+1})$, enabling decisions that are robust to epistemic uncertainty.

\section{Synthetic experiment}
\label{sec:syn-exp}

We study covariate shift in regression with $p=5$ covariates, where the shift acts only through the marginal $\mathcal P_X$ (Appendix~\ref{app:experiments} gives full details). We evaluate unweighted predictive systems (CMCPS/CBIN/CIDR) against shift-aware weighted variants. We consider three weighting regimes: oracle weights (true $w_i^\star$), plug-in odds weights (estimated $\hat w_i$), and CI-odds weights that propagate uncertainty in $\hat w_i$ to produce weight bands.

For all three predictive systems, we adopt a split setup. We first fit a linear regression model on an independent training set, and then hold this model fixed when constructing predictive systems on the calibration set. For CMCPS, we use the linear regression residual as the conformity score (split residual-based CPS). For CIDR, we use the linear regression fitted value $\hat y(x)$ as a one-dimensional summary (single-index model), and apply IDR with the usual order on $\hat y(x)$ using only the calibration sample. For CBIN, we construct bins by clustering the calibration covariates via $k$-nearest-neighbor structure, yielding $10$ clusters, and then apply weighted empirical-CDF binning within the assigned cluster.

\paragraph{Results.}
Each method outputs, for every target test point $x$, a conformal predictive CDF band $\Pi[\mu_{\text{obs}},x;w_{n+1}]$ together with a \textit{crisp} predictive CDF $\widehat{F}(\cdot\mid x)$ obtained by minimizing the worst-case CRPS over the band \citep{allen_-sample_2025}, i.e.,
\begin{equation*}
    \widehat{F}(y \mid x) = \Pi_u[\mu_{\text{obs}},x;w_{n+1}](y)
    - \frac{1}{2}\Pi_u[\mu_{\text{obs}},x;w_{n+1}](y)^2
    + \frac{1}{2}\Pi_\ell[\mu_{\text{obs}},x;w_{n+1}](y)^2 .
\end{equation*}
We evaluate CRPS (lower is better) and the dispersion of PIT values, which equals $1/12\approx 0.0833$ under probabilistic calibration, for the crisp predictive CDFs. In addition, we report marginal coverage of $90\%$ prediction intervals (PIs) constructed by inverting the conformal band: for $\alpha=0.1$,
\begin{equation*}
        I_{\alpha}(x)=\Big[\inf\{y:\Pi_u[\mu_{\text{obs}},x;w_{n+1}](y)\ge \alpha/2\}, 
        \inf\{y:\Pi_\ell[\mu_{\text{obs}},x;w_{n+1}](y)\ge 1-\alpha/2\}\Big],
\end{equation*}
which yields a conservative PI that is valid for any CDF consistent with $\Pi[\mu_{\text{obs}},x;w_{n+1}]$.

Table~\ref{tab:syn-exp-results} summarizes performance on the target domain under covariate shift for three GCPS (CMCPS, CBIN, CIDR) and four weighting schemes (none, oracle weights, estimated weights, and CI-based weights that propagate estimation uncertainty). Across all CPSs, the unweighted methods under-cover (coverage $\approx 0.83$-$0.88$), confirming that ignoring covariate shift leads to miscalibrated uncertainty on the target distribution. In contrast, weighting restores validity: oracle and estimated weighting bring coverage close to (and slightly above) the nominal 0.90 level, while CI-based weighting is most conservative (coverage up to $\approx 0.97$). This increased coverage comes with the expected trade-off in sharpness, reflected in slightly higher CRPS for the CI-based variants and larger predictive bands (cf.\ the more conservative PIT dispersions). At the same time, CRPS evaluation based on the CRPS-optimal \emph{crisp} CDF should be interpreted with care, especially for the CI-based envelopes, because collapsing the conformal band to a single distribution discards the epistemic uncertainty encoded by the band width. Thus, while the crisp forecast is convenient for summary evaluation, it provides only a partial view of predictive performance. Future work should therefore develop scoring and evaluation procedures that explicitly account for this epistemic uncertainty, for example, by assessing performance as increasingly epistemically uncertain test cases are excluded, analogous to selective evaluation curves used in classification \citep{geifman_selective_2017}. We therefore advise reporting crisp forecasts based on the estimated (non-CI) weights, and use the CI-based bands primarily to quantify epistemic uncertainty. Overall, these results demonstrate that covariate-shift weighting can be layered on conformal distributional calibrators to recover calibrated predictive uncertainty under covariate shift.

\begin{figure*}[!ht]
    \centering
    \includegraphics[width=\textwidth]{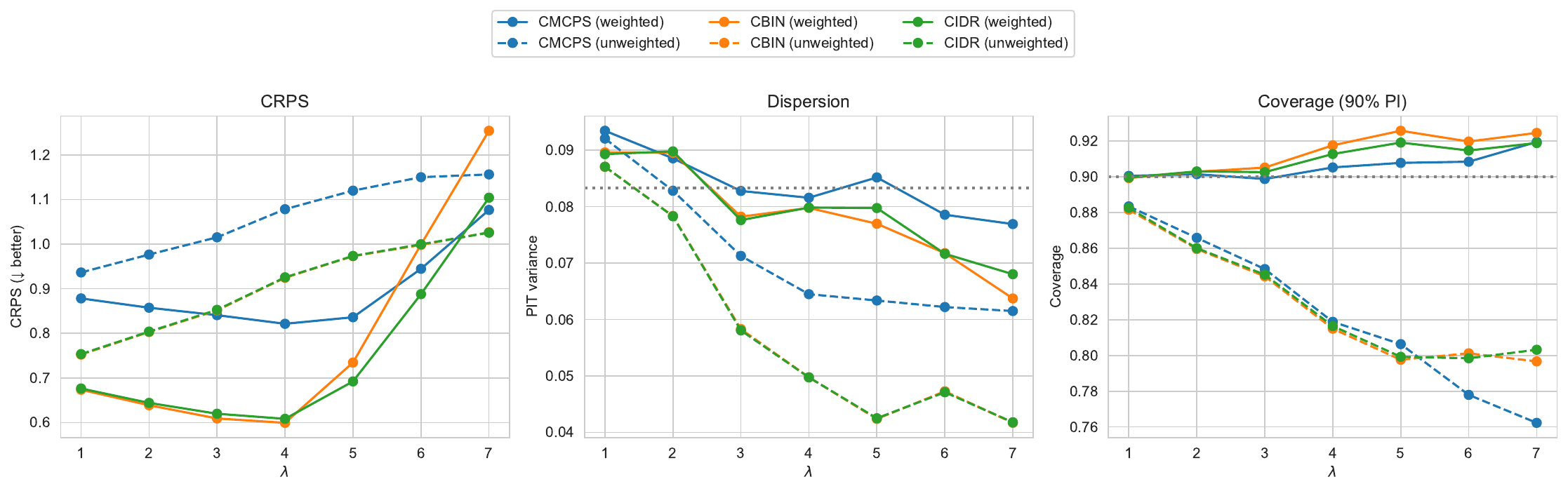}
    \caption{Comparison of three conformal predictive systems (CMCPS, CBIN, CIDR) and their feedback-covariate-shift-aware weighted variants (solid lines) versus unweighted variants (dashed lines) as a function of the design inverse temperature~$\lambda$ in the AAV capsid packaging experiment. The panels show the mean CRPS (left), PIT-based dispersion (middle), and empirical coverage (right) over $T=500$ trials, computed under the design distribution $\mathcal P_{test,\lambda}$. The horizontal line in the coverage panel indicates the nominal target coverage level of $0.9$.}
    \label{fig:bio-design-results}
\end{figure*}

\begin{table}[t]
\centering
\small
\caption{Synthetic 5D covariate-shift regression: CRPS, PIT dispersion, and 90\% prediction-interval coverage on the target domain for CMCPS, CBIN, and CIDR under different weighting schemes (none, oracle, estimated, and CI-based). Point estimates are means over Monte-Carlo replicates; bracketed values are the corresponding $[q_{0.1},\,q_{0.9}]$ quantile intervals.}
\label{tab:syn-exp-results}
\begin{adjustbox}{max width=\textwidth,center}
\begin{tabular}{@{}llccc@{}}
\toprule
GCPS & Weighting scheme & CRPS & Dispersion & Coverage \\
\midrule
\multirow{4}{*}{CMCPS} & No weighting & 1.149 {\scriptsize [1.081,\,1.218]} & 0.0961 {\scriptsize [0.0890,\,0.1021]} & 0.833 {\scriptsize [0.800,\,0.870]} \\
 & Oracle weighting & 1.142 {\scriptsize [1.083,\,1.212]} & 0.0821 {\scriptsize [0.0720,\,0.0915]} & 0.907 {\scriptsize [0.864,\,0.946]} \\
 & Estimated weighting & 1.147 {\scriptsize [1.088,\,1.206]} & 0.0796 {\scriptsize [0.0668,\,0.0901]} & 0.913 {\scriptsize [0.866,\,0.954]} \\
 & Estimated CI weighting & 1.169 {\scriptsize [1.111,\,1.228]} & 0.0616 {\scriptsize [0.0511,\,0.0726]} & 0.953 {\scriptsize [0.928,\,0.984]} \\
\midrule
\multirow{4}{*}{CBIN} & No weighting & 1.113 {\scriptsize [1.048,\,1.167]} & 0.0886 {\scriptsize [0.0813,\,0.0954]} & 0.878 {\scriptsize [0.850,\,0.910]} \\
 & Oracle weighting & 1.093 {\scriptsize [1.034,\,1.150]} & 0.0783 {\scriptsize [0.0697,\,0.0862]} & 0.934 {\scriptsize [0.908,\,0.958]} \\
 & Estimated weighting & 1.096 {\scriptsize [1.038,\,1.153]} & 0.0769 {\scriptsize [0.0677,\,0.0844]} & 0.933 {\scriptsize [0.906,\,0.958]} \\
 & Estimated CI weighting & 1.150 {\scriptsize [1.098,\,1.207]} & 0.0579 {\scriptsize [0.0504,\,0.0654]} & 0.963 {\scriptsize [0.946,\,0.978]} \\
\midrule
\multirow{4}{*}{CIDR} & No weighting & 1.262 {\scriptsize [1.166,\,1.374]} & 0.0900 {\scriptsize [0.0809,\,0.0971]} & 0.860 {\scriptsize [0.826,\,0.900]} \\
 & Oracle weighting & 1.216 {\scriptsize [1.139,\,1.323]} & 0.0750 {\scriptsize [0.0665,\,0.0827]} & 0.939 {\scriptsize [0.915,\,0.966]} \\
 & Estimated weighting & 1.215 {\scriptsize [1.131,\,1.314]} & 0.0743 {\scriptsize [0.0641,\,0.0841]} & 0.937 {\scriptsize [0.910,\,0.964]} \\
 & Estimated CI weighting & 1.262 {\scriptsize [1.178,\,1.372]} & 0.0555 {\scriptsize [0.0467,\,0.0641]} & 0.971 {\scriptsize [0.954,\,0.986]} \\
\bottomrule
\end{tabular}
\end{adjustbox}
\end{table}

\section{Use case: feedback covariate shift in ML-driven biomolecular design}
\label{sec:feedback-use-case}
Adeno-associated viruses (AAVs) are widely used as delivery vectors for gene therapy, and the capsid is the protein shell that encapsulates the viral genome and mediates delivery to target cells. A standard strategy is therefore to engineer large libraries of AAV capsid variants and then select desirable properties such as packaging, infectivity, or cell-type specificity. A major practical bottleneck, however, is that many variants in standard starting libraries already fail at the packaging stage, meaning that they do not successfully assemble into viable capsids carrying the genetic payload. Since packaging is the minimum requirement for any downstream gene-delivery objective, improving packaging fitness while retaining sufficient sequence diversity is a natural and practically important design problem.

We follow the AAV capsid packaging use case of \citet{fannjiang_conformal_2022}, which in turn adopts the library-design methodology of \citet{zhu_optimal_2024}. Concretely, the underlying data come from an AAV5 7-mer peptide-insertion library built from the standard NNK sequence distribution, where each unique sequence is assigned a packaging enrichment score computed from counts before and after a packaging-based selection experiment. \citet{fannjiang_conformal_2022} use these data to study feedback covariate shift (FCS) in an abundant-data regime: they hold out $10^6$ labeled sequences for calibration and test purposes, train a neural network predictor $\hat{\mu}(\cdot)$ on the remaining $7,552,729$ sequences, and then construct designed test distributions $\mathcal P_{\text{test},\lambda}$ indexed by an inverse-temperature parameter $\lambda \in \{1, \ldots, 7\}$. Ideally, these test distributions are exponential tilts of the baseline NNK training distribution $p_{\text{NNK}}$, with mass function $p_{\lambda}(x) \propto p_{\text{NNK}}(x)\exp(\lambda \hat{\mu}(x))$.

So larger values of $\lambda$ increasingly concentrate mass on sequences with high predicted packaging fitness. In practice, following \citet{zhu_optimal_2024}, $\mathcal P_{\text{test},\lambda}$ is implemented by a tractable family of position-wise independent sequence distributions $p_{\phi_\lambda}$, parameterized as independent categorical distributions over the four nucleotides at each of the 21 contiguous sequence positions, corresponding to seven codons of insertion. The resulting covariate shift relative to the baseline NNK distribution is then encoded through the likelihood-ratio weights $w_\lambda(x) \propto \frac{d\mathcal P_{\text{test},\lambda}}{d\mathcal P_{\text{train}}}(x) \approx \frac{p_{\phi_\lambda}(x)}{p_{\text{NNK}}(x)}$.
Thus, increasing $\lambda$ produces designed sequences with higher predicted packaging fitness and stronger shift away from the baseline NNK training distribution $\mathcal P_{\text{train}}$.

In our implementation, we follow this protocol directly. For each $\lambda$, we randomly split the $10^6$ held-out labeled sequences into $990,000$ proposal points and $n_{\text{cal}} = 10,000$ calibration points, sample designed test sequences by rejection sampling from the proposal set, and retain up to $100$ accepted test sequences in each of $T = 500$ trials. We treat the pretrained AAV predictor as fixed by using its sequence-level predictions as inputs to our conformal procedures, and compare weighted and unweighted versions of CMCPS, CBIN, and CIDR on identical calibration/test splits. For CMCPS, we use the residual as conformity score. For CBIN, we construct bins from the calibration-set predicted values using $k=5$ quantile bins, and assign both calibration and test sequences to bins according to these calibration-based quantile edges.

\paragraph{Results.}
For each test sequence, every method returns a conformal CDF band $\Pi[\mu_{\text{obs}},X_{n+1};w_{n+1}]$ and a representative \emph{crisp} CDF $\widehat F(\cdot\mid X_{n+1})$ (as in the synthetic experiment). Figure~\ref{fig:bio-design-results} shows that importance weighting improves performance across the range of design shifts $\lambda$, restoring near-nominal (often slightly conservative) coverage under $p_{\text{test},\lambda}$ compared to the under-covering unweighted variants. At the most extreme shift ($\lambda=7$), weighted CBIN and CIDR can incur slightly worse CRPS, consistent with severe effective-sample size collapse and markedly wider bands, i.e.\ increased epistemic uncertainty. Finally, Figures~\ref{fig:aav_reldiag_lambda7}--\ref{fig:aav_reldiag_lambda1} report threshold-wise reliability for extreme events $\{Y\ge \tau\}$ (high empirical quantiles), confirming that the weighted procedures retain good calibration for rare high-packaging outcomes. Overall, the AAV study demonstrates that FCS-aware weighting can be combined with conformal distributional calibrators to obtain well-calibrated predictive distributions under model-induced (feedback) shifts.

\section{Conclusion}
\label{sec:conclusion}
We extended GCPS beyond exchangeability using permutation weights and introduced weight-uncertainty boxes for robust CDF envelopes with finite-sample or asymptotic confidence guarantees. Efficient rules for CMCPS, CBIN, and CIDR make the method practical, and experiments show calibrated bands that widen under stronger shifts and tighten with more data.

\paragraph{Limitations.} Two caveats deserve emphasis. First,
Theorem~\ref{thm:WU-robust-alpha} requires the weight-uncertainty box to be a \emph{joint} confidence region for $w^\star$; in our experiments this is constructed from a fitted propensity model and the resulting coverage is asymptotic. Second, the closed-form envelopes for CBIN/CMCPS (Proposition~\ref{prop:wurobust-cbin-cmcps}) and CIDR (Proposition~\ref{prop:WU-extrema-antitonic}) assume the active set / IDR order is fixed across $\mathcal W^{\text{WU}}_\alpha$, which restricts the binning, score, and ordering choices to those defined prior to weighting.

\paragraph{Code availability.}
The code accompanying this work will be made publicly available upon publication.

\begin{ack}
    Jef Jonkers is funded by the Research Foundation Flanders (FWO, Ref. 1S11525N).
\end{ack}

% \section*{References}

\bibliographystyle{unsrtnat}
\bibliography{references}

%%%%%%%%%%%%%%%%%%%%%%%%%%%%%%%%%%%%%%%%%%%%%%%%%%%%%%%%%%%%
\newpage
\appendix

\section*{Appendix overview}
This appendix is organized as follows. Appendix \ref{app:related-work} briefly presents some of the related literature. Appendix~\ref{app:permweights} formalizes the permutation-weight representation used to model non-exchangeability. Appendix~\ref{app:proofs} contains proofs of the main theorems and supporting propositions. Appendix~\ref{app:wu-details} collects additional technical details for weight-uncertainty boxes. Appendix~\ref{app:computation} provides full computational derivations for CBIN, CMCPS, and CIDR under weight uncertainty. Appendix~\ref{app:experiments} contains experimental details (synthetic DGP and additional figures).

\section{Related work}
\label{app:related-work}

Generalized conformal predictive systems (GCPS) \citep{allen_-sample_2025}, like classical conformal prediction, obtain out-of-sample calibration guarantees under exchangeability of the sequence $(X_1,Y_1),\ldots,(X_{n+1},Y_{n+1})$. Since exchangeability is frequently violated under distribution shift, substantial effort has gone into extending conformal prediction to non-exchangeable regimes, typically focusing on calibrated prediction sets or intervals rather than full predictive distributions.

A prominent line of work is \emph{weighted conformal prediction} (WCP), which uses importance weights or likelihood-ratio weights to restore coverage under covariate shift \citep{tibshirani_conformal_2019}. More broadly, \citet{barber_conformal_2023} develop conformal methods for certain non-exchangeable settings, while \emph{localized} or \emph{conditional} conformal approaches adapt calibration to local neighborhoods in covariate space \citep{guan_localized_2023,hore_conformal_2024}. Relatedly, \citet{prinster_conformal_2024} study generalized weighted conformal prediction, providing a flexible framework for weighting-based validity beyond classical exchangeability. Finally, \citet{barber_unifying_2025} give a unifying perspective that subsumes many such extensions within a common conformal framework.

Closer to our setting, recent work has started extending conformal \emph{predictive systems} beyond exchangeability in specific shift regimes. In particular, \citet{jonkers_conformal_2024, jonkers_conformal_2025} develop weighted CPS for \emph{known covariate shift} using likelihood-ratio weights, focusing on conformity-measure CPS and probabilistic calibration.

\section{Permutation weights from non-exchangeable sequences}
\label{app:permweights}

Throughout, let $Z_{1:n+1}=(Z_1,\dots,Z_{n+1})$ with $Z_i=(X_i,Y_i)\in\mathcal Z=\mathcal X\times\mathbb R$. Assume $\mathcal Z$ is a standard Borel space. For a realization $z_{1:n+1}\in\mathcal Z^{n+1}$, define
\begin{equation*}
    E_{z_{1:n+1}} := \{ z_1,\dots,z_{n+1}\}
\end{equation*}
assuming that the elements of $Z_{1:n+1}$ are pairwise distinct almost surely. Extensions to possibly repeated observations require more notation but are straightforward. On standard Borel spaces, a regular conditional distribution given the $\sigma$-field generated by $E_{Z_{1:n+1}}$ exists.

Assume that the sequence admits a joint density $f$ on $\mathcal Z^{n+1}$ with respect to a product base measure. Conditional on the unordered multiset, the only remaining randomness is the ordering of the realized atoms. Hence, for a permutation $\sigma\in\mathfrak S_{n+1}$, define
\begin{equation}
    \label{eq:q_sigma_app}
    q_\sigma(z_{1:n+1})
    :=
    \frac{
        f(z_{\sigma(1)},\dots,z_{\sigma(n+1)})
    }{
        \sum_{\tau\in\mathfrak S_{n+1}}
        f(z_{\tau(1)},\dots,z_{\tau(n+1)})
    },
\end{equation}
whenever the denominator is positive. The probability that the atom $z_i$ occupies the test position is obtained by summing over all permutations that place $z_i$ in the last position:
\begin{equation*}
    p_i(z_{1:n+1})
    :=
    \sum_{\sigma\in\mathfrak S_{n+1}:\,\sigma(n+1)=i}
    q_\sigma(z_{1:n+1}),
    \qquad i=1,\dots,n+1 .
\end{equation*}
Thus $p_i(z_{1:n+1})\ge 0$, $\sum_{i=1}^{n+1}p_i(z_{1:n+1})=1$, and
\begin{equation*}
    \mathbb P\{Z_{n+1}=z_i \mid E_{z_{1:n+1}}\}
    =
    p_i(z_{1:n+1}).
\end{equation*}
Consequently, conditional on the unordered multiset, $Z_{n+1}$ is a categorical draw from the realized atoms $\{z_1,\dots,z_{n+1}\}$ with probability weights $(p_1,\dots,p_{n+1})$.

A particularly important special case is covariate shift. Suppose $Z_1,\dots,Z_n$ are drawn from a source/calibration law $\mathcal P_0$, while $Z_{n+1}$ is drawn from a target/test law $\mathcal P_1$, with $\mathcal P_1\ll \mathcal P_0$. If the non-test positions are exchangeable under $\mathcal P_0$, then the likelihood of a permutation whose last atom is $z_i$ is proportional to the likelihood ratio
\begin{equation*}
    r(z_i):=\frac{d\mathcal P_1}{d\mathcal P_0}(z_i).
\end{equation*}
Therefore, the permutation probabilities reduce to
\begin{equation}
    \label{eq:covariate-shift-pi-app}
    p_i(z_{1:n+1})
    =
    \frac{r(z_i)}{\sum_{j=1}^{n+1} r(z_j)} .
\end{equation}
Under covariate shift, the conditional response law is shared, $\mathcal P_1(Y\in\cdot\mid X)=\mathcal P_0(Y\in\cdot\mid X)$, so the joint likelihood ratio depends only on the covariate:
\begin{equation*}
    r(z)=r(x,y)
    =
    \frac{d\mathcal P_{1,X}}{d\mathcal P_{0,X}}(x).
\end{equation*}
Hence the test weight is computable from $X_{n+1}$ alone, without observing $Y_{n+1}$.

When this covariate density ratio is unknown, it can be estimated with a domain classifier. Construct an auxiliary classification problem by pooling source and target covariates and assigning the label $D=0$ to source examples and $D=1$ to target examples. Let $\eta(x):=\mathbb P(D=1\mid X=x)$ denote the classifier's posterior probability. By Bayes' rule, the classifier odds are proportional to the covariate density ratio:
\begin{equation*}
    \frac{\eta(x)}{1-\eta(x)}
    \propto
    \frac{d\mathcal P_{1,X}}{d\mathcal P_{0,X}}(x).
\end{equation*}
The proportionality constant is the ratio of the class priors in the auxiliary classification problem and is independent of $x$. Since the permutation weights are normalized across the atoms, this constant cancels. Thus classifier odds $\eta(x)/(1-\eta(x))$ may be used directly as unnormalized atom weights.

Finally, although the main text presents the permutation weights as probability weights $w_i = p_i$, the GCPS construction depends only on the corresponding weighted empirical law after normalization. Hence, it is sufficient to know the unnormalized atom weights $a_i$ satisfying
\begin{equation*}
    a_i=\lambda p_i,
    \qquad i=1,\dots,n+1,
\end{equation*}
for some common scale factor $\lambda>0$. Indeed,
\begin{equation*}
    P_{\sum_{i=1}^{n+1}a_i\delta_{z_i}}
    =
    P_{\sum_{i=1}^{n+1}p_i\delta_{z_i}},
\end{equation*}
and, for any candidate insertion value $y'$,
\begin{equation*}
    P_{\sum_{i=1}^{n}a_i\delta_{z_i}
      +a_{n+1}\delta_{(x,y')}}
    =
    P_{\sum_{i=1}^{n}p_i\delta_{z_i}
      +p_{n+1}\delta_{(x,y')}} .
\end{equation*}
Thus, likelihood ratios, density ratios, classifier odds, or other importance weights may be used directly without first normalizing them, provided the same scale is used for all $n+1$ atoms.

If the sequence is exchangeable, then $f(z_{\sigma(1)},\dots,z_{\sigma(n+1)})$ is the same for all $\sigma\in\mathfrak S_{n+1}$. Hence $q_\sigma=1/(n+1)!$ and $p_i(z_{1:n+1})=1/(n+1)$ for all $i$, recovering the usual uniform conditional law of the test atom.

\section{Proofs}
\label{app:proofs}

\subsection{Proof of Theorem~\ref{thm:non-exchangeable-known-weights}:  GCPS under known distribution shift}
\label{sec:wcps-proof}

\paragraph{Relation to Theorem~A.1 of \citet{allen_-sample_2025}.}
Our proof follows the same overall strategy as the exchangeable result of \citet{allen_-sample_2025}, but the argument establishing the law of the test point is different. In \citet{allen_-sample_2025}, exchangeability together with deterministic weights implies that, conditional on the augmented empirical measure, the test point is a weighted draw from the observed atoms. In our setting, exchangeability is replaced by the permutation-weight representation from Appendix~\ref{app:permweights}: conditional on the unordered multiset of realized atoms, the test point $Z_{n+1}$ has law $P_{\widetilde\mu}$. Once this is established, we apply the relevant in-sample calibration property of $G$ under the finite-support law $P_{\widetilde\mu}$.

\begin{proof}
    Containment $F^\star\in \Pi[X_{n+1}]$ is immediate from the definition by taking the candidate insertion value $y'=Y_{n+1}$, for which 
    \begin{equation*}
        \mu_{\text{obs}}^{Y_{n+1}} + w^\star_{n+1}(Y_{n+1})\delta_{(X_{n+1},Y_{n+1})} = \widetilde\mu.
    \end{equation*}
    \\  
    \textbf{Step A (item (i)).}
     Fix $\alpha\in(0,1)$. By \eqref{eq:weight-measure} and the tower property,
    \begin{align*}
        \mathbb{P}\big(F^\star(Y_{n+1})\le \alpha\big)
        &= \mathbb{E}\Big[\mathbb{P}\big(G[P_{\widetilde\mu},X_{n+1}](Y_{n+1})\le \alpha
                                  \,\big|\, E_{Z_{1:n+1}}\big)\Big] \\
        &= \mathbb{E}\Big[P_{\widetilde\mu}\big(G[P_{\widetilde\mu},X](Y)\le \alpha\big)\Big] \\
        &\le \alpha,
    \end{align*}
    where the second equality uses that, conditional on $E_{Z_{1:n+1}}$, $(X_{n+1},Y_{n+1})\sim P_{\widetilde\mu}$, and the inequality is in-sample probabilistic calibration of $G$ applied to the finite-support law $P_{\widetilde\mu}$. The same argument with the event $\{G[P_{\widetilde\mu},X](Y-)<\alpha\}$ yields $\alpha \le \mathbb{P}\big(F^\star(Y_{n+1}-)<\alpha\big)$.

    \textbf{Step B (item (ii): isotonic calibration).} For each $y \in \mathbb R$, we need to show that,
    \begin{equation*}
        \mathbb E \Big[ \mathbb 1 \{Y_{n+1} \leq y \} \mid \mathcal{A}(G[P_{\widetilde \mu}, X_{n+1}]) \Big] = G[P_{\widetilde \mu}, X_{n+1}](y).
    \end{equation*}
    By the definition of conditional expectation with respect to a $\sigma$-lattice, see for example \citet[Definition~2.1]{arnold_isotonic_2025}, it suffices to show that for any $A \in \mathcal{A}(G[P_{\widetilde \mu}, X_{n+1}])$ and $B \in \sigma(G[P_{\widetilde \mu}, X_{n+1}])$, it holds that
    \begin{align}
        \mathbb E\!\left[\mathbb 1\{Y_{n+1}\le y\} \mathbb 1_A \right]
        &\le
        \mathbb E\!\left[G[P_{\widetilde \mu}, X_{n+1}](y)\,\mathbb 1_A\right], \label{eq:iso-upper-proof}\\
        \mathbb E\!\left[\mathbb 1\{Y_{n+1}\le y\}\mathbb 1_B\right]
        &=
        \mathbb E\!\left[G[P_{\widetilde \mu}, X_{n+1}](y)\,\mathbb 1_B\right]. \label{eq:iso-borel-proof}
    \end{align}
    Since $Z_{n+1}\mid P_{\widetilde\mu}\sim P_{\widetilde\mu}$ and $G$ is in-sample isotonically calibrated for any finite-support law, applied here to $P_{\widetilde\mu}$, we have that
     \begin{align*}
        \mathbb E\!\left[\mathbb 1\{Y_{n+1}\le y\} \mathbb 1_A \mid P_{\widetilde \mu} \right]
        &\le
        \mathbb E\!\left[G[P_{\widetilde \mu}, X_{n+1}](y)\,\mathbb 1_A \mid P_{\widetilde \mu} \right],\\
        \mathbb E\!\left[\mathbb 1\{Y_{n+1}\le y\}\mathbb 1_B \mid P_{\widetilde \mu}\right]
        &=
        \mathbb E\!\left[G[P_{\widetilde \mu}, X_{n+1}](y)\,\mathbb 1_B \mid P_{\widetilde \mu} \right]. 
    \end{align*}
    Taking expectations yields \eqref{eq:iso-upper-proof} and \eqref{eq:iso-borel-proof}.

   \textbf{Step C (item (iii): auto-calibration).}
   Conditional on $P_{\tilde{\mu}}$ it holds that $(X_{n+1},Y_{n+1})\sim P_{\widetilde\mu}$. Since $G$ is in-sample auto-calibrated under every finite-support law, applied here to $P_{\widetilde\mu}$, we have
 \begin{equation*}
        \mathbb E\!\left[\mathbb 1\{Y_{n+1}\le y\}\mid F^\star,P_{\widetilde\mu}\right]= \mathbb E\!\left[\mathbb 1\{Y_{n+1}\le y\}\mid G[P_{\widetilde\mu},X_{n+1}],P_{\widetilde\mu}\right]
       = G[P_{\widetilde\mu},X_{n+1}](y) = F^\star(y),
   \end{equation*}
   almost surely for all $y$. Taking the conditional expectation with respect to $G[P_{\widetilde\mu},X_{n+1}]$ yields the claim. 
\end{proof}

\subsection{Proof of Theorem~\ref{thm:WU-robust-alpha}}
\label{app:WU-proof}

\begin{proof} \qquad \\
    \medskip
    \noindent\textbf{Part (a).} The event $A_\alpha := \{w^* \in \mathcal{W}_\alpha^{WU}\}$ implies the event inside the probability operator of ~\eqref{eq:WU-envelope-fs}, and $\mathbb{P}(A_\alpha)\ge 1-\alpha$ by~\eqref{eq:WU-CI-fs-app}, hence
    \begin{equation}
        \label{eq:WU-envelope-fs}
        \mathbb{P}\Big(
            F^\star \in \Pi^{\text{WU},\alpha}
        \Big)
        \;\ge\; \mathbb{P}(A_\alpha)
        \;\ge\; 1-\alpha.
    \end{equation}
    \medskip
    \noindent\textbf{Part (b).} is identical, replacing~\eqref{eq:WU-CI-fs-app} by~\eqref{eq:WU-CI-as-app} and taking the $\liminf_{n\to\infty}$.
\end{proof}

\section{Additional technical details for weight-uncertainty boxes}
\label{app:wu-details}

\subsection{(Uncertain) bounds on weights and feasible laws}
\label{app:wubox_feasible}
Let $\hat w_i=\hat w_i(Z_{1:n+1})\ge 0$ be estimated (possibly misspecified) weights attached to the atoms $Z_i$. Let $L_i,U_i$ be given with $0\le L_i\le U_i<\infty$. The true (unknown) unnormalized weights $w^\star=(w^\star_1,\dots,w^\star_{n+1})$ are assumed to satisfy the coordinatewise multiplicative bounds
\begin{equation}
    \label{eq:WU-unnorm-app}
    L_i\,\hat w_i \;\le\; w^\star_i \;\le\; U_i\,\hat w_i, \qquad i=1,\dots,n+1,
\end{equation}
with confidence level $1-\alpha$.

This defines the weight-uncertainty box (same as~\eqref{eq:W_WU})
\begin{equation*}
    \mathcal{W}^{\text{WU}}_\alpha
    := \Big\{\,w\in \mathbb{R}_+^{n+1}:\ L_i\hat w_i \le w_i \le U_i\hat w_i\ \text{for all }i\, \Big\},
\end{equation*}
and the induced feasible probability-weight set
\begin{equation}
    \label{eq:P_WU_app}
    \mathcal{P}^{\text{WU}}_\alpha
    := \Big\{\,p^w \in [0,1]^{n+1}:\ p^w_i=\frac{w_i}{\sum_{j=1}^{n+1} w_j},\ w\in\mathcal W^{\text{WU}}_\alpha \Big\}.
\end{equation}

\subsection{Coverage regimes}
\label{app:wubox_coverage}
We distinguish two regimes for interpreting $\mathcal{W}^{\text{WU}}_\alpha$ as a confidence region for $w^\star$:

\begin{itemize}
    \item \textbf{Finite-sample coverage.}
    \begin{equation}
        \label{eq:WU-CI-fs-app}
        \mathbb{P}\big(w^\star \in \mathcal{W}^{\text{WU}}_\alpha\big) \;\ge\; 1-\alpha
    \end{equation}
    for each fixed sample size $n$.
    
    \item \textbf{Asymptotic coverage.}
    \begin{equation}
        \label{eq:WU-CI-as-app}
        \liminf_{n\to\infty} \mathbb{P}\big(w^\star \in \mathcal{W}^{\text{WU}}_\alpha\big) \;\ge\; 1-\alpha .
    \end{equation}
\end{itemize}

In both regimes, $\mathcal{W}^{\text{WU}}_\alpha$ provides simultaneous coverage for the entire vector $w^\star$: the event $\{w^\star \in \mathcal{W}^{\text{WU}}_\alpha\}$ means that all coordinates lie within their respective bounds simultaneously. The hyper-rectangular form $L_i\hat w_i \le w_i \le U_i\hat w_i$ is chosen for analytical convenience; the method used to obtain $(L_i,U_i)$ (e.g.\ inversion of tests, pivotal arguments, or resampling) is immaterial for our results as long as~\eqref{eq:WU-CI-fs-app} or~\eqref{eq:WU-CI-as-app} holds.

\section{Computation under weight-uncertainty boxes}
\label{app:computation}
This appendix derives efficient evaluation of the weight-robust envelopes $\Pi^{\text{WU}}_\ell,\Pi^{\text{WU}}_u$ for CBIN and CMCPS (CIDR is treated separately in Appendix~\ref{app:wurobust-idr}).

\subsection{CBIN and CMCPS: closed forms from box-extrema}
\label{app:wurobust-binning}

\paragraph{Setup and assumptions (CBIN/CMCPS).}
Fix the test covariate $x:=X_{n+1}$ and write $l_i:=L_i\hat w_i$ and $u_i:=U_i\hat w_i$ for $i=1,\dots,n+1$. Let $I(x)\subseteq\{1,\dots,n\}$ denote the set of calibration indices actually used to form the prediction at $x$. Depending on the method, $I(x)$ may depend on $x$ and on any quantities fixed prior to the optimization over $w\in\mathcal W^{\text{WU}}$ (for example, a split-trained model, a precomputed partition, a neighborhood structure, or a fixed score map). For the proposition below, we assume:
\begin{enumerate}
    \item[\textbf{A1}] The active set $I(x)$ is the same for all $w\in\mathcal W^{\text{WU}}$, i.e., for CMCPS, the score map $S$ used at $x$ is also fixed over $\mathcal W^{\text{WU}}$.
    \item[\textbf{A2}] Denominator positivity: for all $w\in\mathcal W^{\text{WU}}$,
    $w_{n+1}+\sum_{i\in I(x)}w_i>0$.
\end{enumerate}

\begin{proposition}[Box-extrema yield closed-form robust envelopes for CBIN/CMCPS]
    \label{prop:wurobust-cbin-cmcps}
    Define, for $y\in\mathbb R$,
    \begin{equation*}
        I_1(y):=\{i\in I(x):Y_i\le y\},\qquad I_0(y):=\{i\in I(x):Y_i>y\}.
    \end{equation*}
    
    \emph{(a) CBIN.} Suppose $G$ is the weighted empirical CDF on $I(x)$ with insertion of the test atom:
    \begin{equation*}
        G\!\left[P_{\mu_{\text{obs}}(w)+w_{n+1}\delta_{(x,y')}},x\right](y)
        =
        \frac{\sum_{i\in I(x)} w_i\,\mathbb 1\{Y_i\le y\}+w_{n+1}\,\mathbb 1\{y'\le y\}}
        {\sum_{i\in I(x)} w_i+w_{n+1}}.
    \end{equation*}
    Then the weight-robust envelopes satisfy, for all $y$,
    \begin{align}
        \Pi^{\text{WU}}_u(y)
        &=
        \frac{ u_{n+1}+\sum_{i\in I_1(y)} u_i}
        { u_{n+1}+\sum_{i\in I_1(y)} u_i+\sum_{i\in I_0(y)} l_i},
        \label{eq:app-WU-cbin-upper}
        \\
        \Pi^{\text{WU}}_\ell(y)
        &=
        \frac{\sum_{i\in I_1(y)} l_i}
        { u_{n+1}+\sum_{i\in I_1(y)} l_i+\sum_{i\in I_0(y)} u_i}.
        \label{eq:app-WU-cbin-lower}
    \end{align}
    In particular, $\Pi^{\text{WU}}_\ell,\Pi^{\text{WU}}_u$ are piecewise constant with breakpoints at $\{Y_i:i\in I(x)\}$.
    
    \emph{(b) CMCPS.} If instead $G$ is a (split) conformity-measure CPS that is representable as the same weighted empirical CDF but applied to \emph{scores} $s_i:=S(Z_i)$ and $s_y:=S((x,y))$, and $y\mapsto s_y$ is nondecreasing for fixed $x$, then \eqref{eq:app-WU-cbin-upper}--\eqref{eq:app-WU-cbin-lower} hold verbatim with the events $\{Y_i\le y\}$ replaced by $\{s_i\le s_y\}$ (equivalently, with $Y_i$ replaced by $s_i$ and $y$ replaced by $s_y$).
\end{proposition}

\begin{proof}
    Fix $y\in\mathbb R$. In CBIN,
    \begin{equation*}
        \sup_{y'} \mathbb 1\{y'\le y\}=1\quad(\text{attained by any }y'\le y),
        \qquad
        \inf_{y'} \mathbb 1\{y'\le y\}=0\quad(\text{attained by any }y'> y).
    \end{equation*}
    Thus the inner $\sup/\inf$ over $y'$ reduces to optimizing, over $w\in\mathcal W^{\text{WU}}$,
    \begin{equation*}
        \phi_u(w;y):=\frac{w_{n+1}+\sum_{i\in I_1(y)} w_i}{w_{n+1}+\sum_{i\in I_1(y)} w_i+\sum_{i\in I_0(y)} w_i},
        \qquad
        \phi_\ell(w;y):=\frac{\sum_{i\in I_1(y)} w_i}{w_{n+1}+\sum_{i\in I_1(y)} w_i+\sum_{i\in I_0(y)} w_i}.
    \end{equation*}
    A direct monotonicity check gives: $\phi_u(\cdot;y)$ is increasing in $w_{n+1}$ and in $w_i$ for $i\in I_1(y)$, and decreasing in $w_i$ for $i\in I_0(y)$; hence its maximum is attained at $w_{n+1}=u_{n+1}$, $w_i=u_i$ for $i\in I_1(y)$, and $w_i=l_i$ for $i\in I_0(y)$, yielding \eqref{eq:app-WU-cbin-upper}. Similarly, $\phi_\ell(\cdot;y)$ is increasing in $w_i$ for $i\in I_1(y)$ and decreasing in $w_{n+1}$ and in $w_i$ for $i\in I_0(y)$; hence its minimum is attained at $w_{n+1}=u_{n+1}$, $w_i=l_i$ for $i\in I_1(y)$, and $w_i=u_i$ for $i\in I_0(y)$, yielding \eqref{eq:app-WU-cbin-lower}. (Notice $w_{n+1}=u_{n+1}$ in both envelopes: increasing the insertion weight widens the band by raising the upper envelope and lowering the lower envelope.)
    
    For CMCPS under the stated score representation, the same argument applies after replacing $\mathbb 1\{Y_i\le y\}$ by $\mathbb 1\{s_i\le s_y\}$. Monotonicity of $y\mapsto s_y$ ensures the inner $\sup/\inf$ is again attained by score-extreme insertions, and the same box-extrema argument yields the claimed formulas.
\end{proof}

\begin{remark}[When the active set depends on $w$]
    The closed-form expressions above rely on the active set $I(x)$ (and, for CMCPS, the score map $S$) being fixed over the uncertainty box. If these objects depend on the candidate weight vector $w$, then the optimization over $\mathcal W^{\text{WU}}$ must also account for changes in the active set, and the simple box-extrema formulas need not hold globally. In such cases, one may still apply the formulas on regions of $\mathcal W^{\text{WU}}$ where the active set is constant.
\end{remark}

\paragraph{Implementation note.}
After sorting the breakpoints $\{Y_i:i\in I(x)\}$ (or scores $\{s_i:i\in I(x)\}$), \eqref{eq:app-WU-cbin-upper}--\eqref{eq:app-WU-cbin-lower} can be evaluated for all breakpoints using prefix sums of $\{l_i\}$ and $\{u_i\}$ in $O(k)$ time after an $O(k\log k)$ sort, where $k:=|I(x)|$.

\subsection{Conformal isotonic distributional regression (CIDR): box-extremal weights}
\label{app:wurobust-idr}
We specialize the weight-robust envelope to conformal IDR, where $G$ is defined via antitonic regression of indicators $\mathbb{1}\{Y_i\le y\}$ with respect to the order used by IDR. For a given test covariate $x:=X_{n+1}$ and threshold $y\in\mathbb R$, let $(X'_1,Y'_1),\dots,(X'_{n+1},Y'_{n+1})$ be the ordered pairs used by the IDR fit (augmented with the test atom). Define $T_i(y):=\mathbb{1}\{Y'_i\le y\}$.

For an ordered weight vector $w'$ feasible under the box constraints, the IDR predictive CDF at the test position $i^\star$ can be written as
\begin{equation*}
    f_{i^\star}(w',y)
    := \min_{s\le i^\star}\ \max_{t\ge i^\star}
    \frac{\sum_{r=s}^t \mathbb{1}\{Y'_r\le y\}\,w'_r}{\sum_{r=s}^t w'_r},
\end{equation*}
ignoring intervals of zero total weight. The weight-robust IDR envelopes optimize this functional over the feasible weight box.

\begin{proposition}[Weight-extrema for antitonic functionals]
    \label{prop:WU-extrema-antitonic}
    Fix $y\in\mathbb{R}$ and let $Y'_1,\dots,Y'_{n+1}$ be responses in the fixed order used by antitonic regression. Let $\hat w'_i\ge 0$ be the corresponding estimated weights, and define the feasible ordered box
    \begin{equation*}
        \mathcal{W}^{\text{WU}'}
        := \Big\{w'\in\mathbb{R}_+^{n+1}: L'_i \hat w'_i \le w'_i \le U'_i \hat w'_i\ \text{for all }i\Big\}.
    \end{equation*}
    Define, for $w'\in\mathcal{W}^{\text{WU}'}$ and index $j$,
    \begin{equation*}
        f_j(w', y)
        := \min_{s\le j}\ \max_{t\ge j}
        \frac{\sum_{r=s}^t \mathbb{1}\{Y'_r\le y\} w'_r}{\sum_{r=s}^t w'_r},
    \end{equation*}
    ignoring intervals of zero total weight. Define box-extreme weights
    \begin{equation*}
         w_i^{\prime(+)} :=
        \begin{cases}
        U'_i \hat w'_i, & Y'_i\le y \\
        L'_i \hat w'_i, & Y'_i> y ,
        \end{cases}
        \qquad
        w_i^{\prime(-)} :=
        \begin{cases}
        L'_i \hat w'_i, & Y'_i\le y,\\
        U'_i \hat w'_i, & Y'_i> y.
        \end{cases}
    \end{equation*}
    Then, for every $j$,
    \begin{equation*}
        \inf_{w'\in\mathcal{W}^{\text{WU}'}} f_j(w', y)
        = f_j\big(w^{\prime(-)},y\big),
        \qquad
        \sup_{w'\in\mathcal{W}^{\text{WU}'}} f_j(w',y)
        = f_j\big(w^{\prime(+)},y\big).
    \end{equation*}
    Moreover, $f_j(\cdot,y)$ is invariant under positive rescaling of $w'$.
\end{proposition}

\begin{proof}
    Fix $y\in\mathbb{R}$, the ordered responses $Y'_1,\dots,Y'_{n+1}$. For $w'\in\mathcal{W}^{\text{WU}'}$ and indices $s\le t$, write
    \begin{equation*}
        A_{s,t}(w', y) := \frac{\sum_{r=s}^t \mathbb{1}\{Y'_r\le y\} w'_r}{\sum_{r=s}^t w'_r},
    \end{equation*}
    with the convention that intervals with $\sum_{r=s}^t w'_r = 0$ are ignored. Then, for any fixed $j$,
    \begin{equation*}
        f_j(w',y) = \min_{s\le j} \ \max_{t\ge j} A_{s,t}(w', y).
    \end{equation*}
    
    \emph{(i) Scale invariance.}
    For any $c>0$,
    \begin{equation*}
        A_{s,t}(c w', y) = A_{s,t}(w', y),
    \end{equation*}
    hence $f_j(c w',y)=f_j(w',y)$. Therefore optimizing $f_j$ over $\mathcal{W}^{\text{WU}'}$ or over the corresponding normalized probability weights is equivalent.
    
    \emph{(ii) Coordinatewise monotonicity of $A_{s,t}$.}
    Fix an interval $[s,t]$ and a coordinate $k$.

    If $k\notin[s,t]$, then $A_{s,t}(w', y)$ does not depend on $w'_k$.

    If $k\in[s,t]$ and $\sum_{r=s}^t w'_r>0$, writing
    \begin{equation*}
        N := \sum_{r=s}^t T_r w'_r, \qquad D := \sum_{r=s}^t w'_r,
    \end{equation*}
    we have
    \begin{equation*}
        \frac{\partial A_{s,t}}{\partial w'_k} = \frac{T_k D - N}{D^2}.
    \end{equation*}
    If $T_k=1$, then
    \begin{equation*}
        \frac{\partial A_{s,t}}{\partial w'_k}
        = \frac{1 - A_{s,t}}{D} \ge 0,
    \end{equation*}
    so $A_{s,t}$ is increasing in $w'_k$.
    If $T_k=0$, then
    \begin{equation*}
        \frac{\partial A_{s,t}}{\partial w'_k}
        = -\frac{N}{D^2} \le 0,
    \end{equation*}
    so $A_{s,t}$ is decreasing in $w'_k$.

    If $\sum_{r=s}^t w'_r=0$ and we increase a single $w'_k$ in $[s,t]$, the interval becomes active with value $A_{s,t}(w', y)=T_k\in\{0,1\}$, which is consistent with the same monotonicity conclusions. Thus, for every interval $[s,t]$:
    \begin{equation*}
        \begin{aligned}
            & Y'_k\le y,\ w'_k \le \tilde w'_k \ \Longrightarrow\ A_{s,t}(w',y) \le A_{s,t}(\tilde w',y),\\
            &Y'_k> y,\ w'_k \ge \tilde w'_k \ \Longrightarrow\ A_{s,t}(w',y) \le A_{s,t}(\tilde w',y).
        \end{aligned}
    \end{equation*}

    \emph{(iii) Monotonicity of $f_j$.}
    For each fixed $s\le j$, define
    \begin{equation*}
        g_s(w', y) := \max_{t\ge j} A_{s,t}(w', y).
    \end{equation*}
    Being the pointwise maximum of $\{A_{s,t}\}_{t\ge j}$, $g_s$ inherits the same coordinatewise monotonicity:
    \begin{equation*}
        \begin{aligned}
            &Y'_k\le y,\ w'_k \le \tilde w'_k \ \Longrightarrow\ g_s(w',y) \le g_s(\tilde w',y),\\
            &Y'_k> y,\ w'_k \ge \tilde w'_k \ \Longrightarrow\ g_s(w',y) \le g_s(\tilde w',y).
        \end{aligned}
    \end{equation*}
    Then
    \begin{equation*}
        f_j(w',y) = \min_{s\le j} g_s(w', y)
    \end{equation*}
    is the pointwise minimum of the $\{g_s\}_{s\le j}$ and therefore satisfies, for any two weight vectors that agree in all other coordinates,
    \begin{equation*}
        \begin{aligned}
            &Y'_k\le y,\ w'_k \le \tilde w'_k \ \Longrightarrow\ f_j(w',y) \le f_j(\tilde w',y),\\
            &Y'_k> y,\ w'_k \ge \tilde w'_k \ \Longrightarrow\ f_j(w',y) \le f_j(\tilde w',y).
        \end{aligned}
    \end{equation*}
    
    \emph{(iv) Extremizers on the WU box.}
    Since
    \(
      \mathcal{W}^{\text{WU}'}
      = \prod_{i=1}^{n+1} [L'_i \hat w'_i,\, U'_i \hat w'_i],
    \)
    the vectors $w^{\prime(+)}$ and $w^{\prime(-)}$ are, respectively, the maximal and minimal elements of $\mathcal{W}^{\text{WU}'}$ that increase coordinates with $Y'_i\le y$ and decrease coordinates with $Y'_k> y$.

    By the monotonicity from (iii), for every fixed $j$ and every $w'\in\mathcal{W}^{\text{WU}'}$,
    \begin{equation*}
        f_j\big(w^{\prime(-)},y\big)
        \le f_j(w',y)
        \le f_j\big(w^{\prime(+)},y\big),
    \end{equation*}
    which yields
    \begin{equation*}
        \inf_{w'\in\mathcal{W}^{\text{WU}'}} f_j(w', y)
        = f_j\big(w^{\prime(-)},y\big),
        \qquad
        \sup_{w'\in\mathcal{W}^{\text{WU}'}} f_j(w',y)
        = f_j\big(w^{\prime(+)},y\big).
    \end{equation*}
\end{proof}

\section{Experimental details and additional figures}
\label{app:experiments}

\subsection{Synthetic experiment: data-generating process}
\label{app:syn-dgp}
We consider a pure covariate shift regression problem with $p=5$ covariates. First, we draw a large pool from a mildly correlated Gaussian base distribution $\mathcal{P}_X$, specifically,
\begin{equation*}
    X \sim \mathcal N(0, \Sigma), \qquad \Sigma = (1-a^2) I_5 + a^2 \mathbf{1}_5 \mathbf{1}_5^T, \qquad a = 0.2
\end{equation*}
so that each coordinate has unit variance and each pair of distinct coordinates has correlation $a^2$. Note that $\mathbf{1}_5$ denotes the $5$-dimensional vector of ones and $I_5$ denotes the $5 \times 5$ identity matrix. 

Next, we assign the domain membership through
\begin{equation*}
    S \mid X = x \sim \text{Bernoulli}(\pi(x)), \quad \text{where } \pi(x) = 0.1 + 0.8 \, \text{expit}(\lambda \beta^T x),
\end{equation*}
with
\begin{equation*}
    \beta = (1.2, -1.0, 0.8, -0.6, 0.4)^T \quad \text{and} \quad \lambda = 2.5.
\end{equation*}
We interpret $S=0$ as the source domain and $S=1$ as the target domain. Hence, the source and target covariate laws are
\begin{equation*}
    \mathcal P_X^{(0)}=\mathcal L(X\mid S=0),
    \qquad
    \mathcal P_X^{(1)}=\mathcal L(X\mid S=1),
\end{equation*}
which differ whenever $\pi(x)$ depends nontrivially on $x$.

Conditional on $X=x$, generate the response as
\begin{equation*}
    Y=\mu(x)+\varepsilon_Y,\qquad \varepsilon_Y\sim\mathcal N(0,1),
\end{equation*}
independently of $S$ given $X$, with
\begin{equation*}
    \mu(x)=2\sin(x_1x_2)+1.2\tanh(x_3)+0.8x_4^2+0.6\max(0,x_5).
\end{equation*}

Therefore,
\begin{equation*}
    \mathcal P_{Y\mid X,S}= \mathcal P_{Y\mid X}.
\end{equation*}
Thus, the conditional outcome law is the same in both domains, and only the marginal covariate distribution shifts because the selection probability $\pi(x)$ depends on $ x$. Conditioning on $S$ tilts the base covariate distribution differently in the source and target domains.

To obtain a dataset of the desired size, we first generate an oversized pool of candidate observations and then subsample without replacement $n_{\text{source}}=1000$ points with $S=0$ and $n_{\text{target}}=500$ points with $S=1$. This is then repeated over 100 Monte Carlo replicates to produce the results in Table \ref{tab:syn-exp-results}.

\subsection{Setup for the conformal CDF-band visualizations}
\label{app:syn-cdf-bands}
Figures~\ref{fig:cdf-band-example} and \ref{fig:cdf-band-weight-example} are generated using the synthetic covariate-shift model from Appendix~\ref{app:syn-dgp}. In both figures, $S\in\{0,1\}$ denotes domain membership, with $S=0$ the source domain and $S=1$ the target domain, and the importance weights are odds weights
\begin{equation*}
    w_i:= \frac{or(X_i)}{\sum_{j=1}^{n+1} or(X_j)}, \qquad or(x)=\frac{\pi(x)}{1-\pi(x)}, \qquad \pi(x)=\mathbb P(S=1\mid X=x).
\end{equation*}
For Figure~\ref{fig:cdf-band-example}, we first generate one large synthetic dataset with $n_{\text{source}}=100{,}000$ and $n_{\text{target}}=50{,}000$, randomly shuffle it, and then consider nested prefixes of total sizes
\begin{equation*}
    n\in\{100,200,2000,20000\}.
\end{equation*}
Within each prefix, we retain the source-domain observations ($S=0$), split them evenly into a training subset and a calibration subset, and use the first available target-domain observation ($S=1$) as the test point. The column headers report the realized values of $(n_{\text{train}},n_{\text{cal}})$, which vary with the number of source observations present in each prefix. The darker band uses point-estimated odds weights, while the lighter band uses confidence-interval-based odds weights obtained by propagating uncertainty from the fitted logistic propensity model.

For Figure~\ref{fig:cdf-band-weight-example}, we fix a single synthetic dataset generated with $n_{\text{source}}=1000$, $n_{\text{target}}=500$, shift-strength parameter $\lambda=2.5$. The source-domain observations are split evenly into training and calibration subsets. We then select four target-domain test points according to the empirical quantiles of the true importance weights $w(x)$ among target-domain observations, namely the $0$, $0.25$, $0.75$, and $1$ quantiles (using nearest-quantile selection). Thus, this figure keeps the sample size fixed and varies only the degree of covariate shift of the displayed target test point.

In both figures, the black curve is the oracle conditional CDF $y\mapsto \Phi\!\left(y-\mu(x)\right)$ at the displayed test point, where $\mu(x)$ is given by the data-generating mechanism.

\subsection{Compute resources}
\label{app:compute}
All experiments are computationally light and were run on a single Apple M1~Pro CPU (no GPU, no cluster). The synthetic experiment (Section~\ref{sec:syn-exp}) and the AAV use case (Section~\ref{sec:feedback-use-case}) each complete in well under an hour of wall-clock time per configuration. The dominant costs are the closed-form box-extrema for CBIN/CMCPS ($O(k\log k)$ per test point) and two PAVA evaluations per threshold for CIDR (Appendix~\ref{app:computation}). For the AAV packaging predictor $\hat\mu(\cdot)$, which we use as-is rather than retrain, we refer to the compute description in \citet{fannjiang_conformal_2022} and \citet{zhu_optimal_2024}.

\subsection{Additional figures}
\label{app:additional-figures}
\begin{figure*}[!b]
    \centering
    \includegraphics[width=0.9\textwidth]{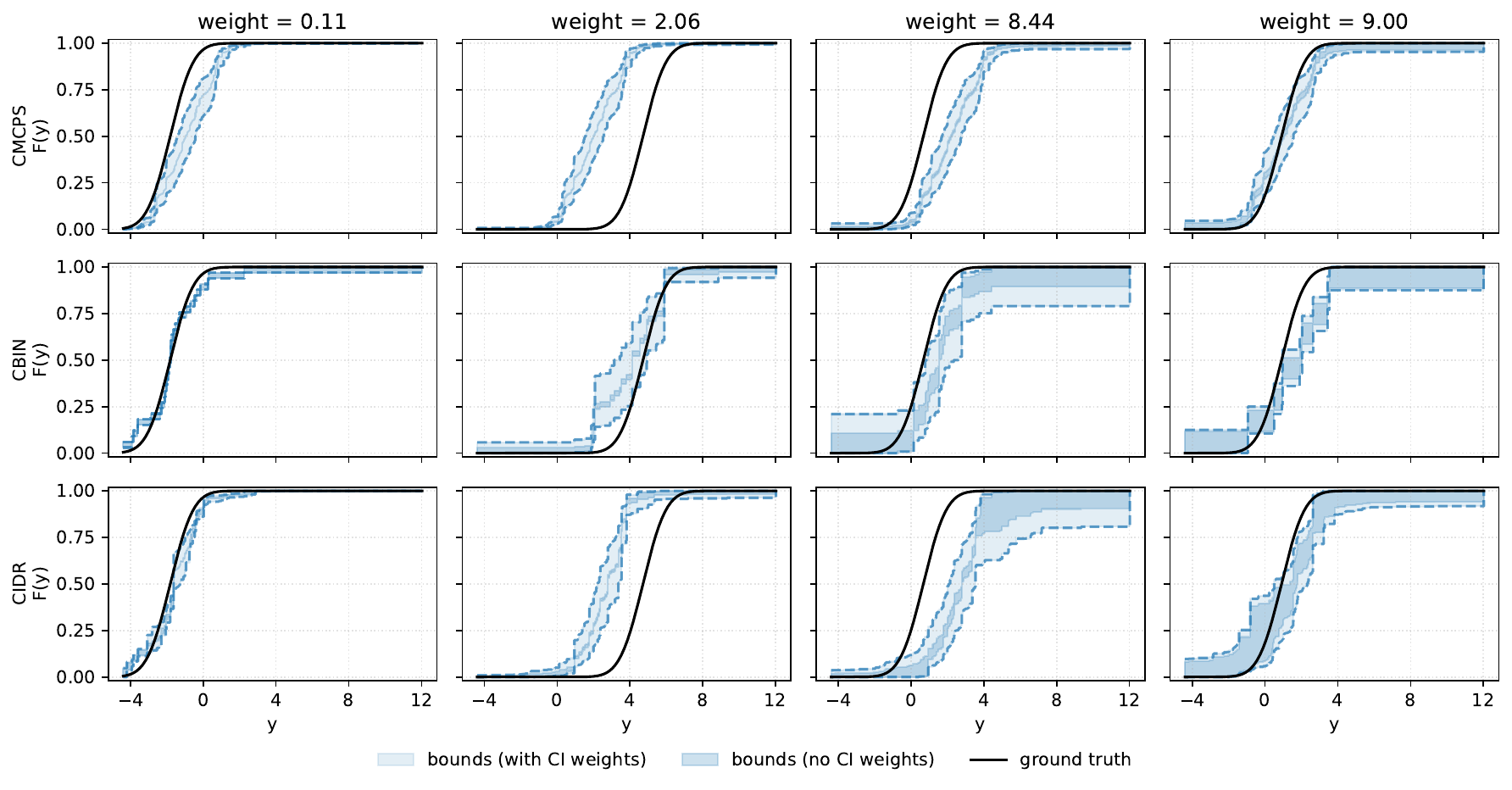}
    \caption{Effect of importance weight on conformal CDF-band thickness under estimated covariate shift. Using the synthetic setup described in Appendix~\ref{app:syn-dgp}, we plot weighted conformal predictive CDF bands for four target-domain test points as the true importance weight increases. Results are shown for CMCPS, CBIN, and CIDR. The lighter region uses confidence-interval-based odds weights, while the darker region uses point-estimated odds weights. As the importance weight increases, the bands widen, indicating greater epistemic uncertainty for target-domain test points that are less well supported by the source-domain training and calibration data. The black curve denotes the oracle conditional CDF.}
    \label{fig:cdf-band-weight-example}
\end{figure*}

\begin{figure}
  \centering
  \begin{subfigure}[b]{0.8\textwidth}
    \centering
    \includegraphics[width=\linewidth]{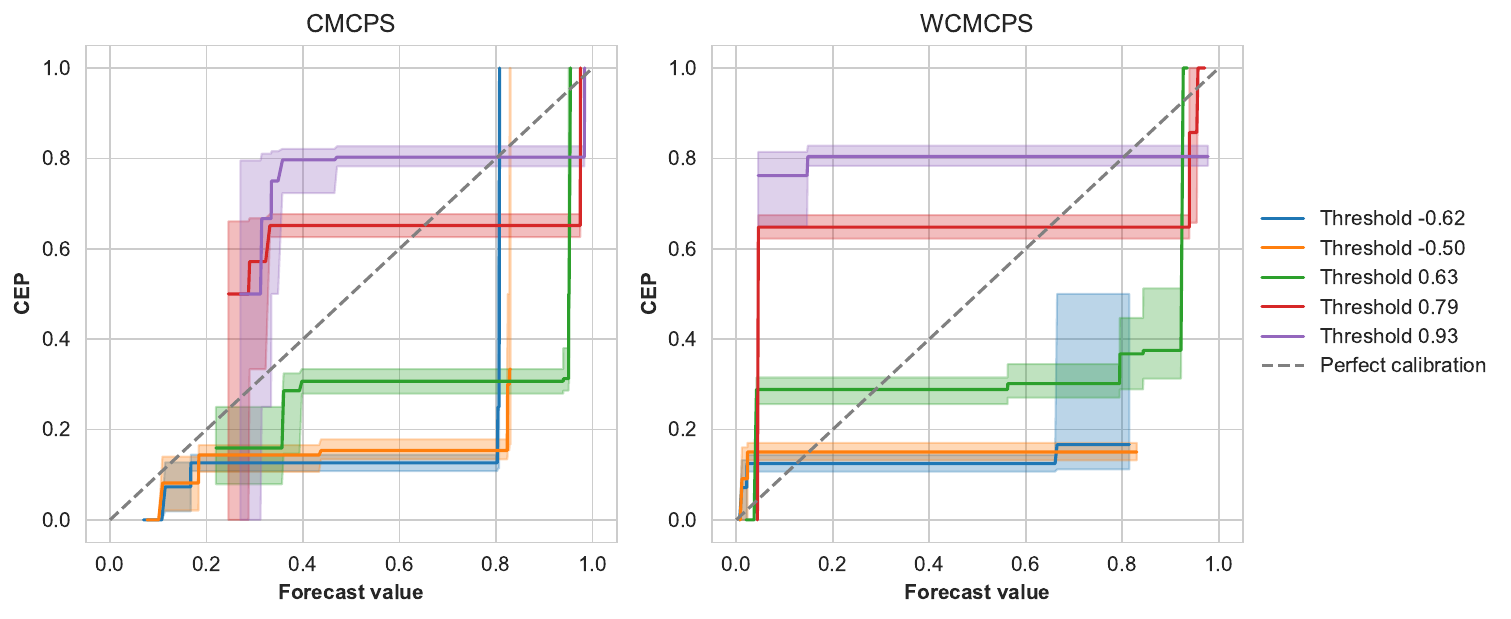}
    \caption{CMCPS / WCMCPS}
    \label{fig:aav_reldiag_lambda7_cmcps}
  \end{subfigure}\hfill
  \begin{subfigure}[b]{0.8\textwidth}
    \centering
    \includegraphics[width=\linewidth]{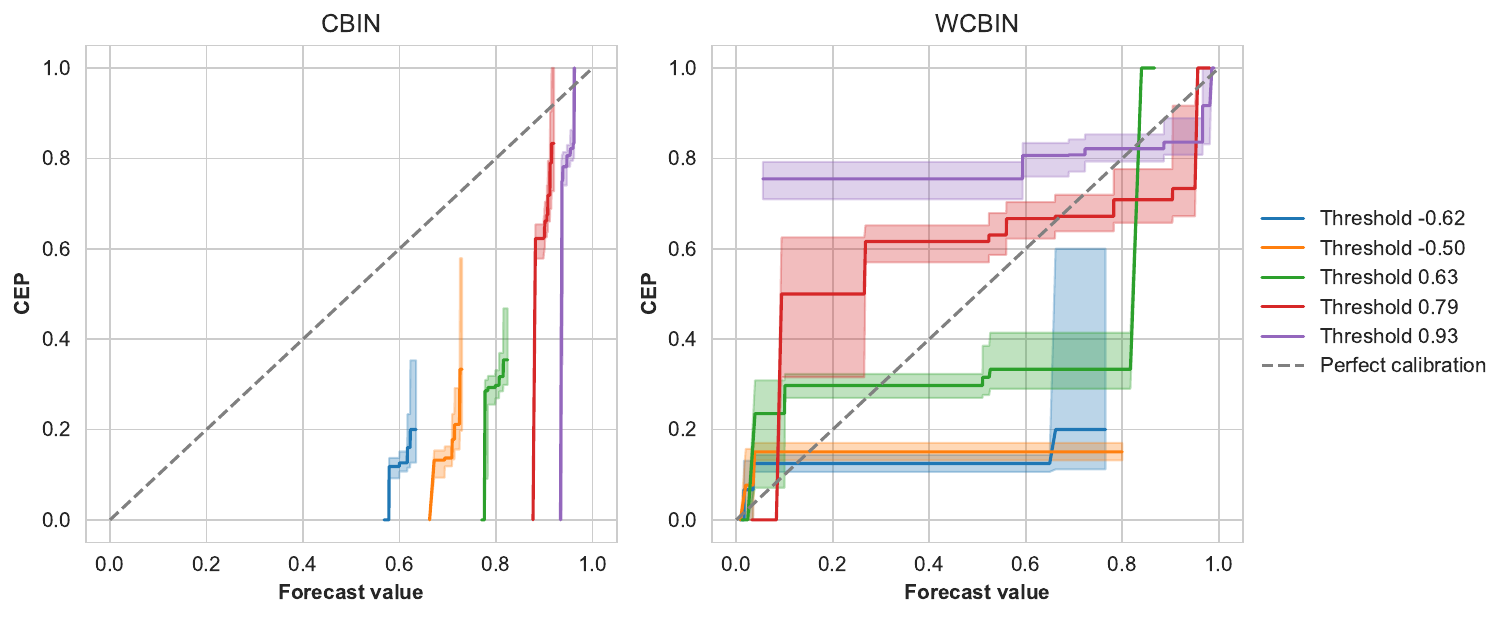}
    \caption{CBIN / WCBIN}
    \label{fig:aav_reldiag_lambda7_cbin}
  \end{subfigure}\hfill
  \begin{subfigure}[b]{0.8\textwidth}
    \centering
    \includegraphics[width=\linewidth]{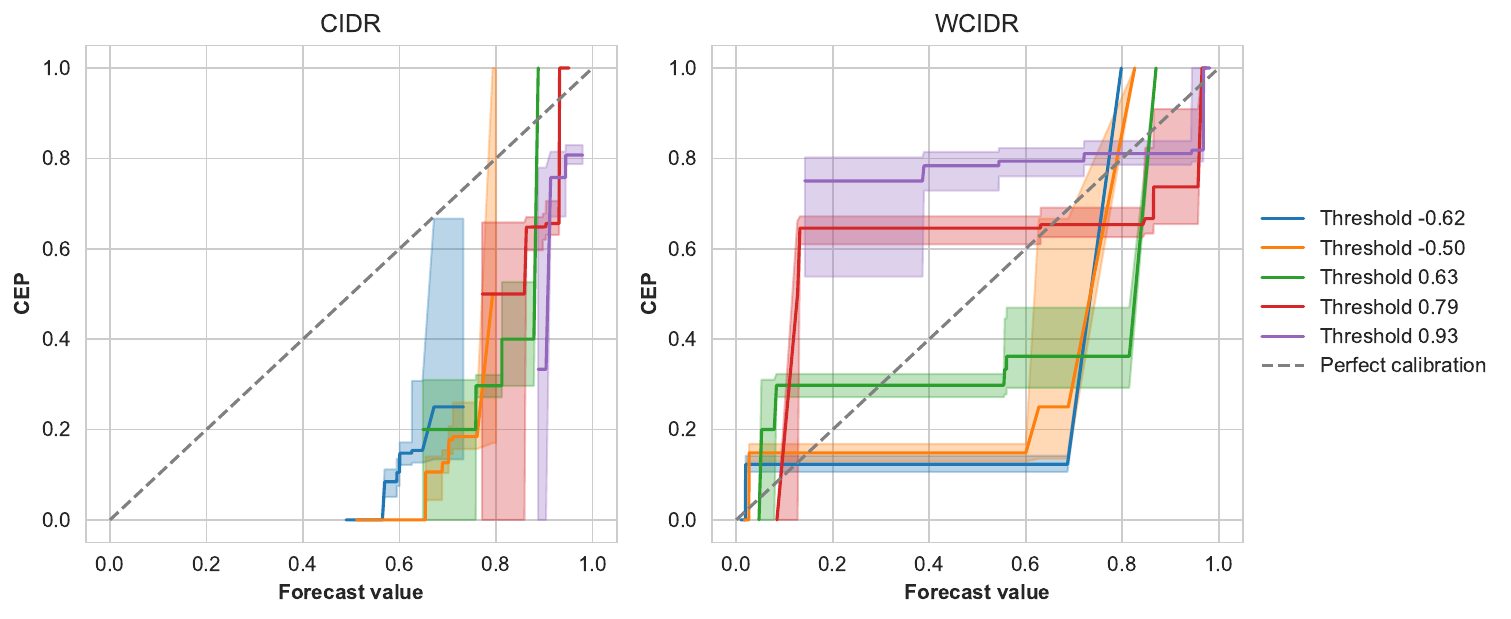}
    \caption{CIDR / WCIDR}
    \label{fig:aav_reldiag_lambda7_cidr}
  \end{subfigure}
  \caption{CORP reliability diagrams \citep{dimitriadis_stable_2021}. The x-axis shows predicted probabilities, and the y-axis shows the observed frequency of the outcome. A perfectly calibrated model lies along the diagonal: predictions match observed outcomes. Curves above the diagonal indicate underestimation, while curves below indicate overestimation. Threshold-wise CORP reliability diagrams for the AAV design task at inverse temperature $\lambda = 7$. Each panel shows weighted (right) and unweighted (left) calibration curves across several extreme packaging thresholds, comparing FCS-aware calibration (W*) with its unweighted counterpart. The shaded areas represent 95\% confidence intervals for the reliability diagram, constructed using bootstrapping.}
  \label{fig:aav_reldiag_lambda7}
\end{figure}

\begin{figure}
  \centering
  \begin{subfigure}[b]{0.8\textwidth}
    \centering
    \includegraphics[width=\linewidth]{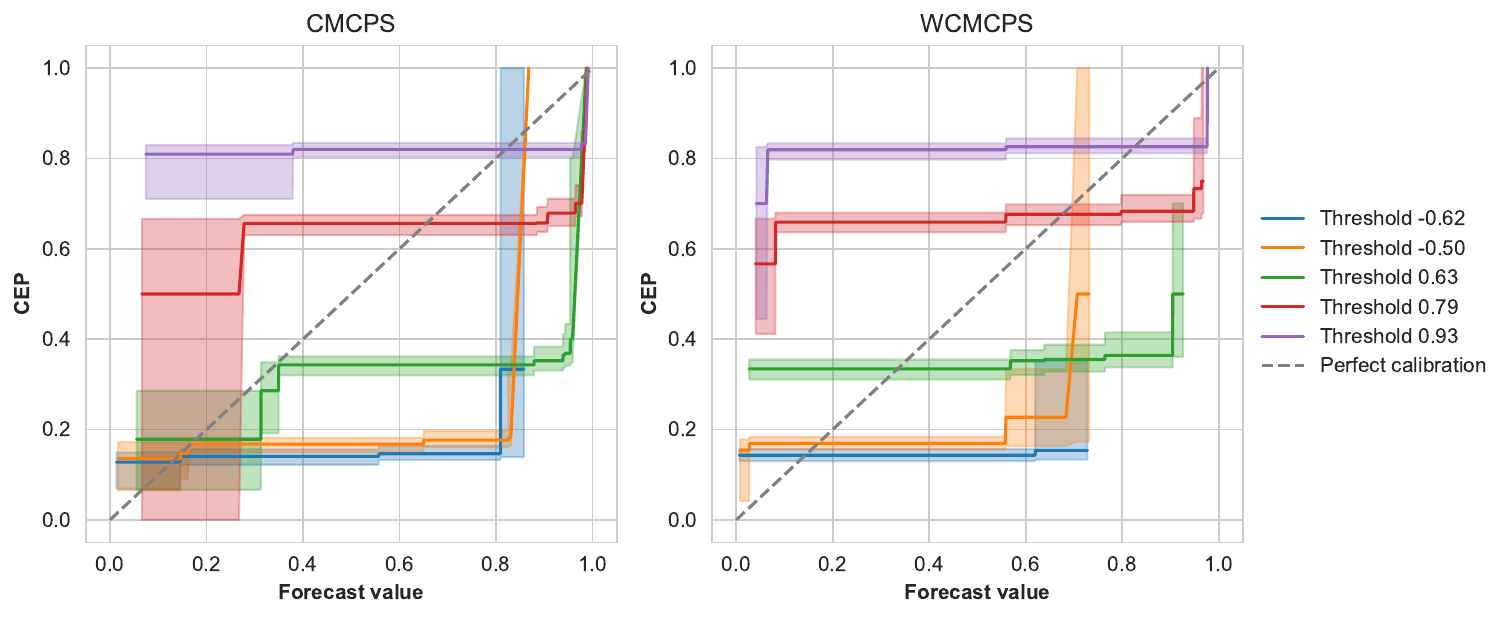}
    \caption{CMCPS / WCMCPS}
    \label{fig:aav_reldiag_lambda6_cmcps}
  \end{subfigure}\hfill
  \begin{subfigure}[b]{0.8\textwidth}
    \centering
    \includegraphics[width=\linewidth]{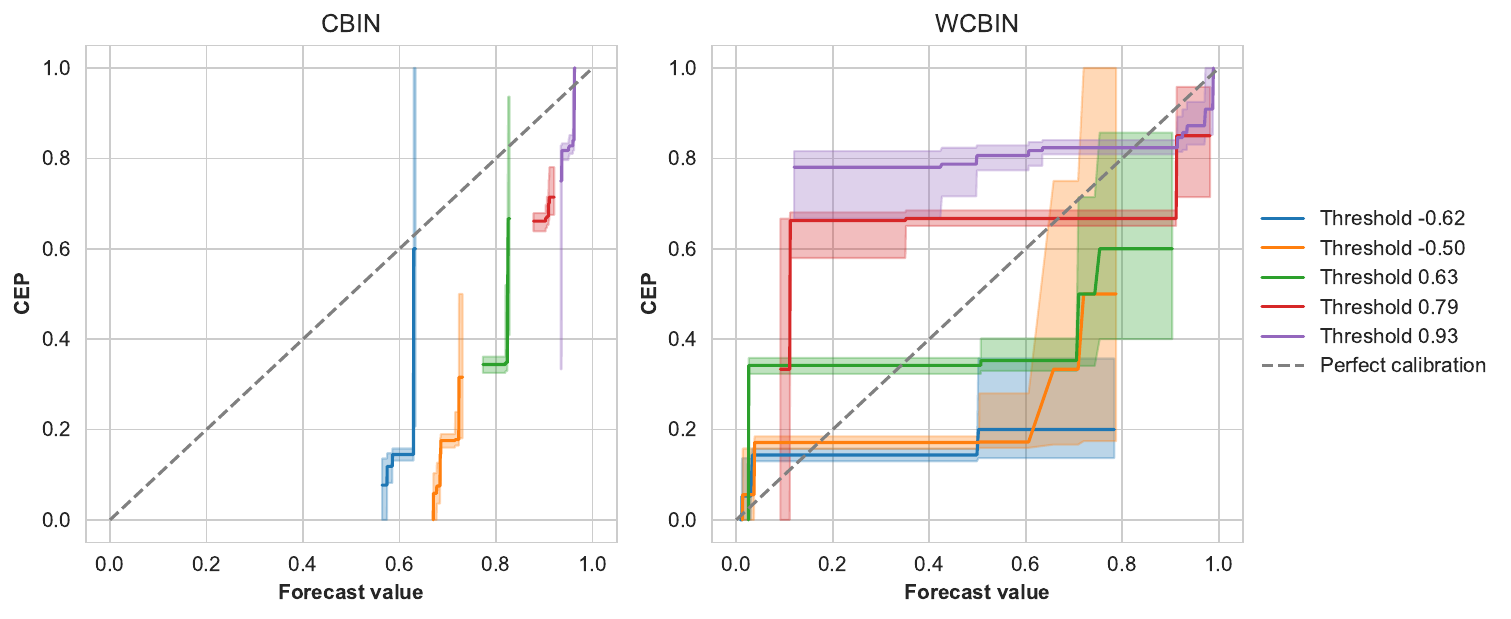}
    \caption{CBIN / WCBIN}
    \label{fig:aav_reldiag_lambda6_cbin}
  \end{subfigure}\hfill
  \begin{subfigure}[b]{0.8\textwidth}
    \centering
    \includegraphics[width=\linewidth]{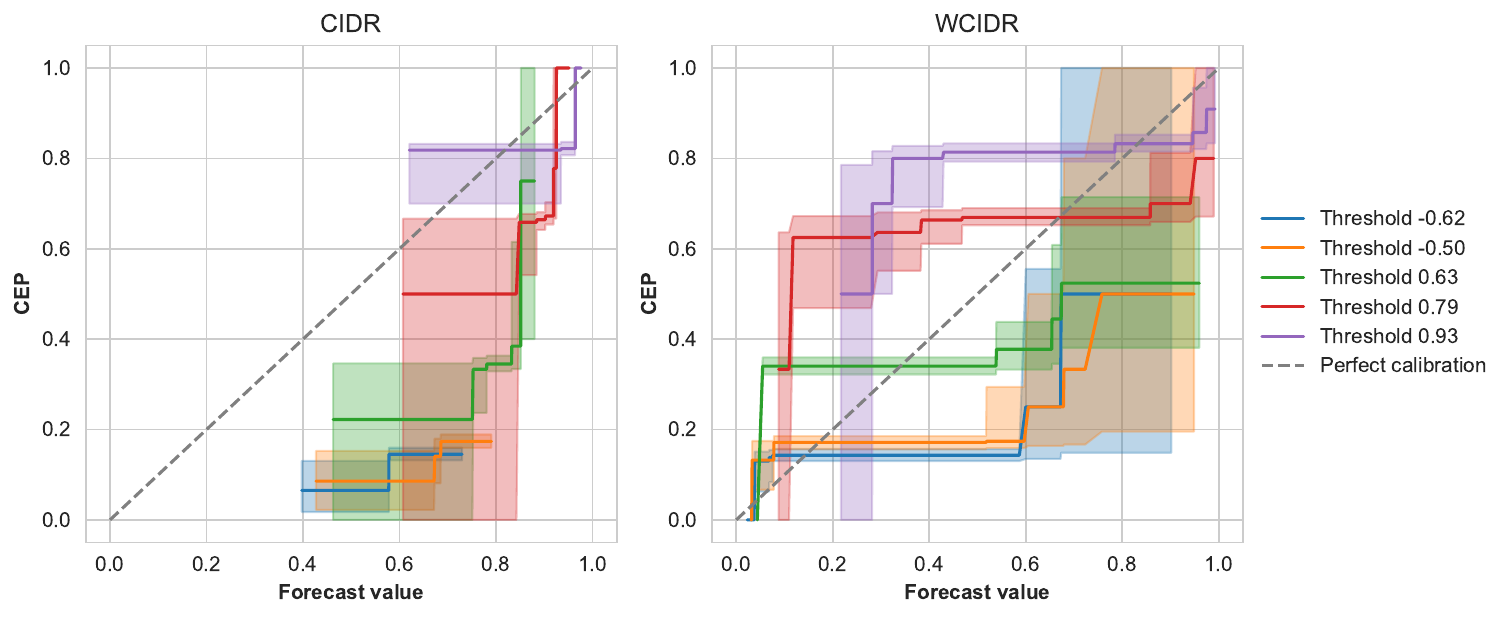}
    \caption{CIDR / WCIDR}
    \label{fig:aav_reldiag_lambda6_cidr}
  \end{subfigure}
  \caption{CORP reliability diagrams \citep{dimitriadis_stable_2021}. The x-axis shows predicted probabilities, and the y-axis shows the observed frequency of the outcome. A perfectly calibrated model lies along the diagonal: predictions match observed outcomes. Curves above the diagonal indicate underestimation, while curves below indicate overestimation. Threshold-wise CORP reliability diagrams for the AAV design task at inverse temperature $\lambda = 6$. Each panel shows weighted (right) and unweighted (left) calibration curves across several extreme packaging thresholds, comparing FCS-aware calibration (W*) with its unweighted counterpart. The shaded areas represent 95\% confidence intervals for the reliability diagram, constructed using bootstrapping.}
  \label{fig:aav_reldiag_lambda6}
\end{figure}

\begin{figure}
  \centering
  \begin{subfigure}[b]{0.8\textwidth}
    \centering
    \includegraphics[width=\linewidth]{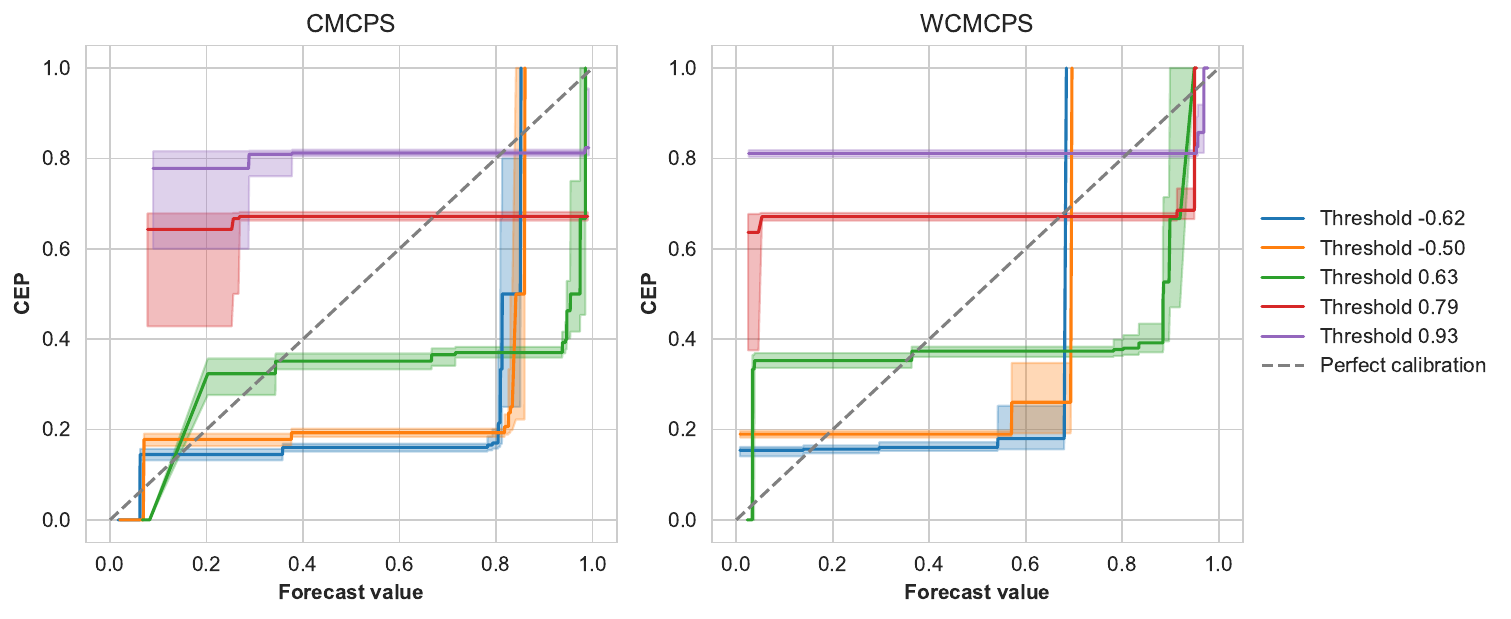}
    \caption{CMCPS / WCMCPS}
    \label{fig:aav_reldiag_lambda5_cmcps}
  \end{subfigure}\hfill
  \begin{subfigure}[b]{0.8\textwidth}
    \centering
    \includegraphics[width=\linewidth]{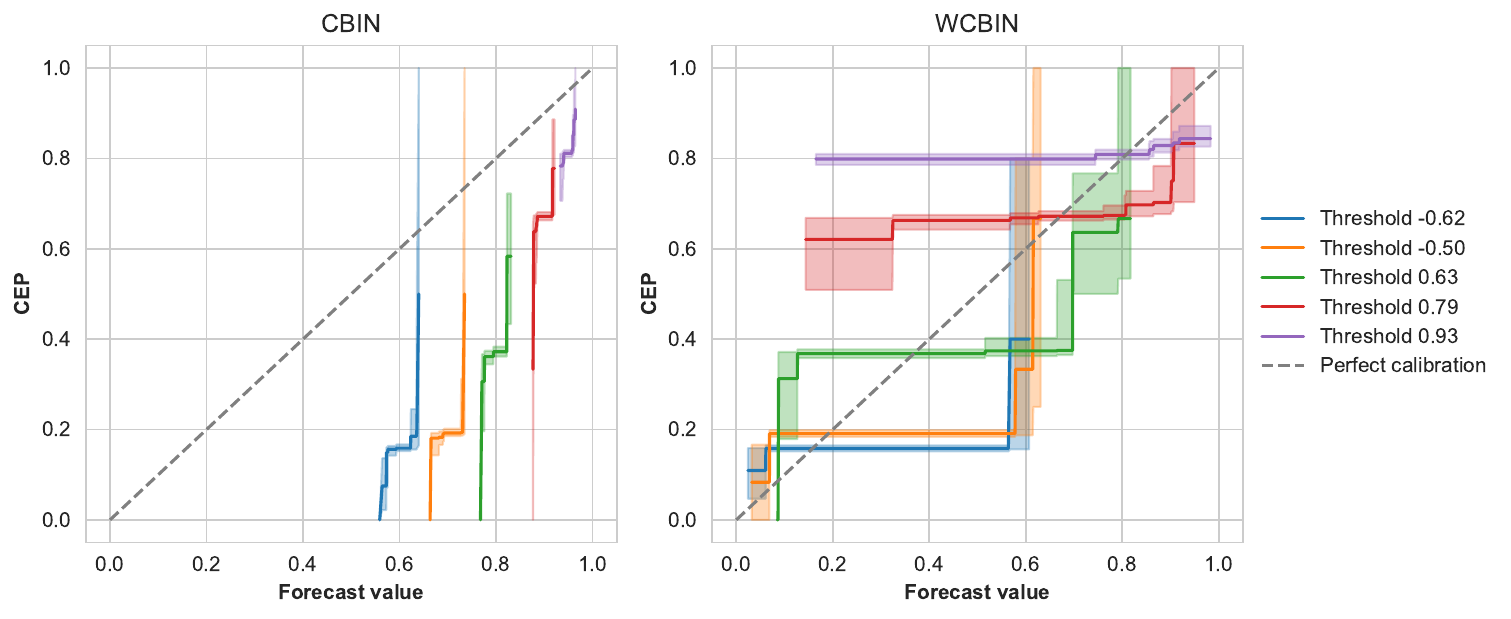}
    \caption{CBIN / WCBIN}
    \label{fig:aav_reldiag_lambda5_cbin}
  \end{subfigure}\hfill
  \begin{subfigure}[b]{0.8\textwidth}
    \centering
    \includegraphics[width=\linewidth]{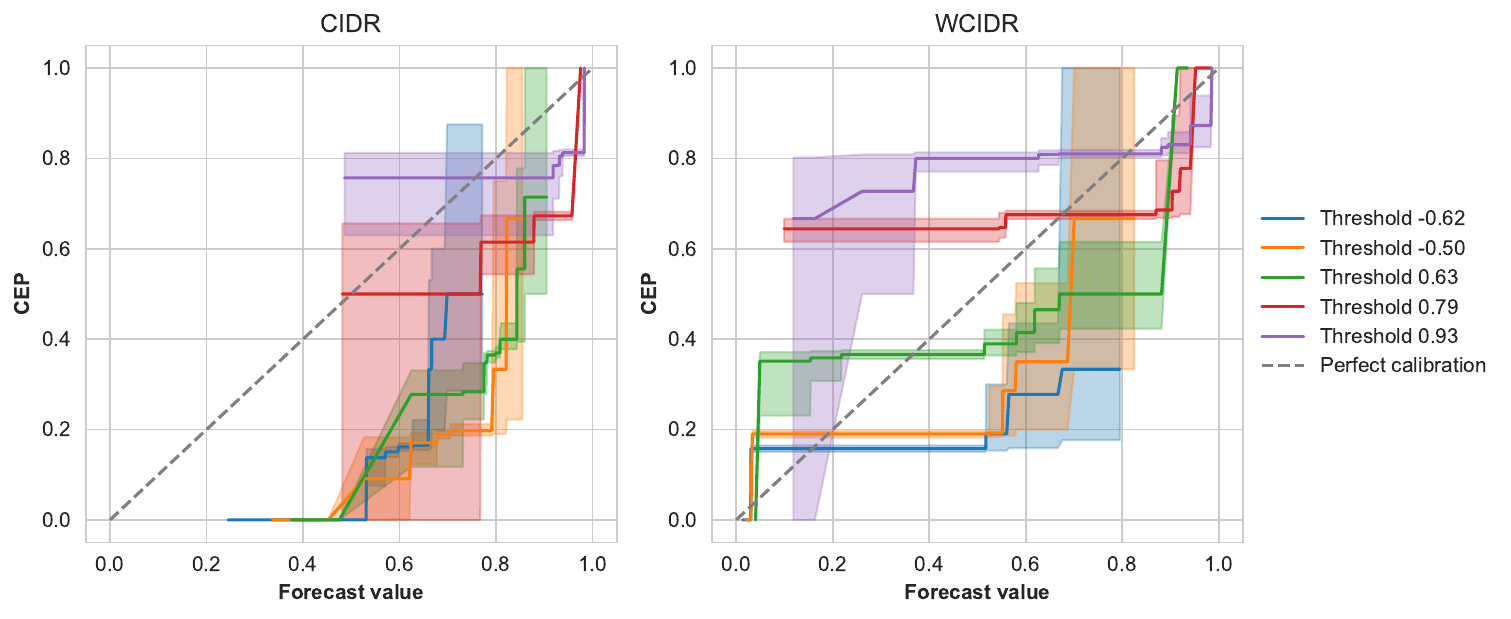}
    \caption{CIDR / WCIDR}
    \label{fig:aav_reldiag_lambda5_cidr}
  \end{subfigure}
  \caption{CORP reliability diagrams \citep{dimitriadis_stable_2021}. The x-axis shows predicted probabilities, and the y-axis shows the observed frequency of the outcome. A perfectly calibrated model lies along the diagonal: predictions match observed outcomes. Curves above the diagonal indicate underestimation, while curves below indicate overestimation. Threshold-wise CORP reliability diagrams for the AAV design task at inverse temperature $\lambda = 5$. Each panel shows weighted (right) and unweighted (left) calibration curves across several extreme packaging thresholds, comparing FCS-aware calibration (W*) with its unweighted counterpart. The shaded areas represent 95\% confidence intervals for the reliability diagram, constructed using bootstrapping.}
  \label{fig:aav_reldiag_lambda5}
\end{figure}

\begin{figure}
  \centering
  \begin{subfigure}[b]{0.8\textwidth}
    \centering
    \includegraphics[width=\linewidth]{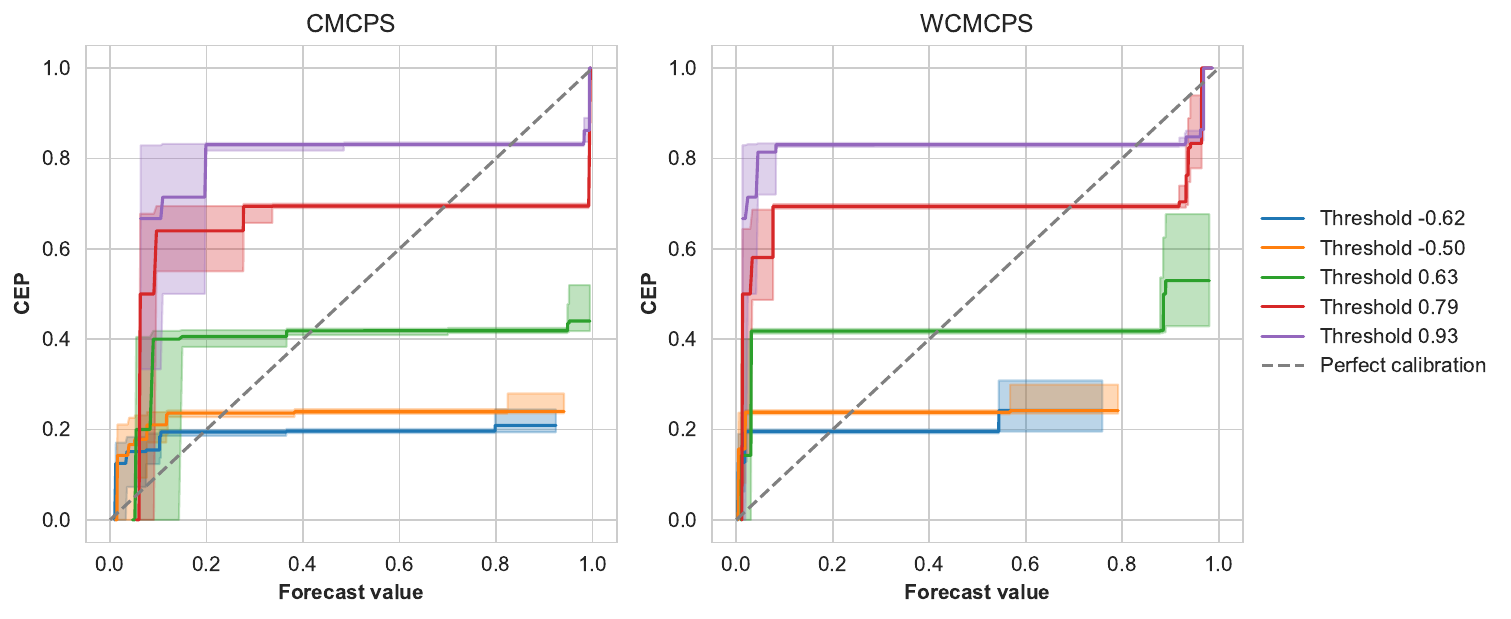}
    \caption{CMCPS / WCMCPS}
    \label{fig:aav_reldiag_lambda4_cmcps}
  \end{subfigure}\hfill
  \begin{subfigure}[b]{0.8\textwidth}
    \centering
    \includegraphics[width=\linewidth]{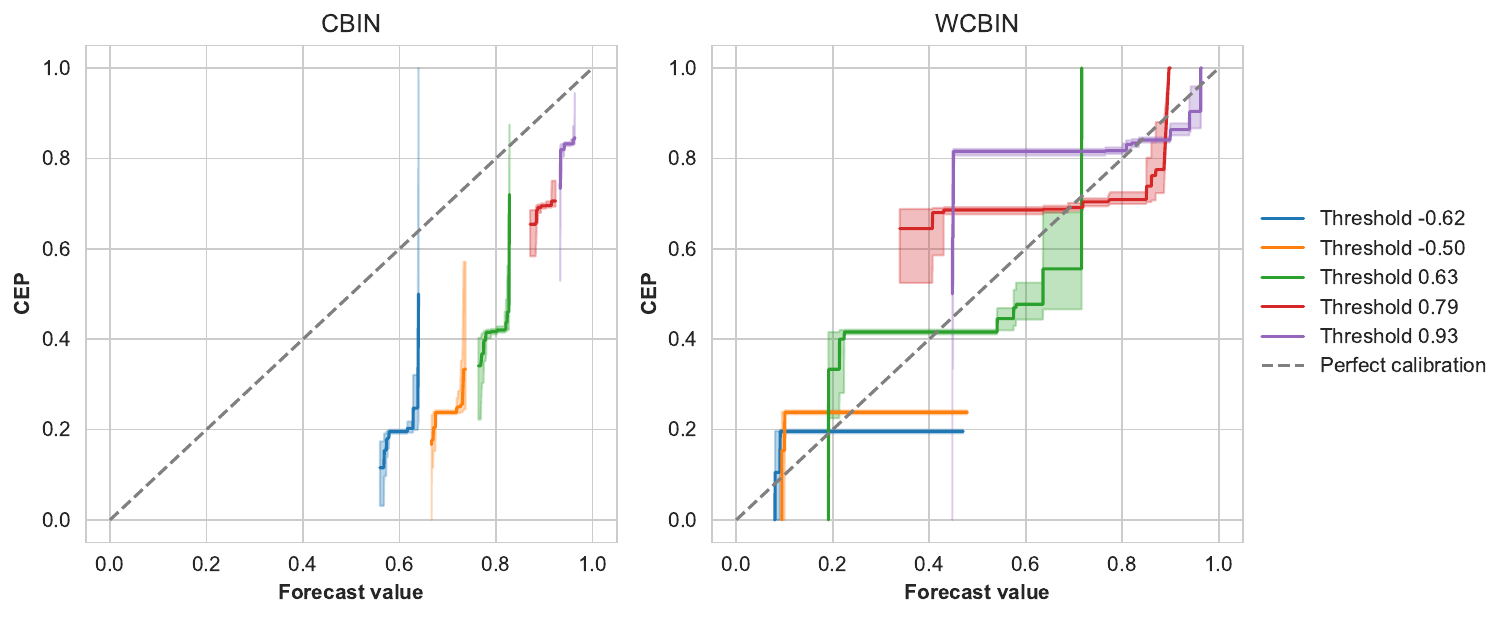}
    \caption{CBIN / WCBIN}
    \label{fig:aav_reldiag_lambda4_cbin}
  \end{subfigure}\hfill
  \begin{subfigure}[b]{0.8\textwidth}
    \centering
    \includegraphics[width=\linewidth]{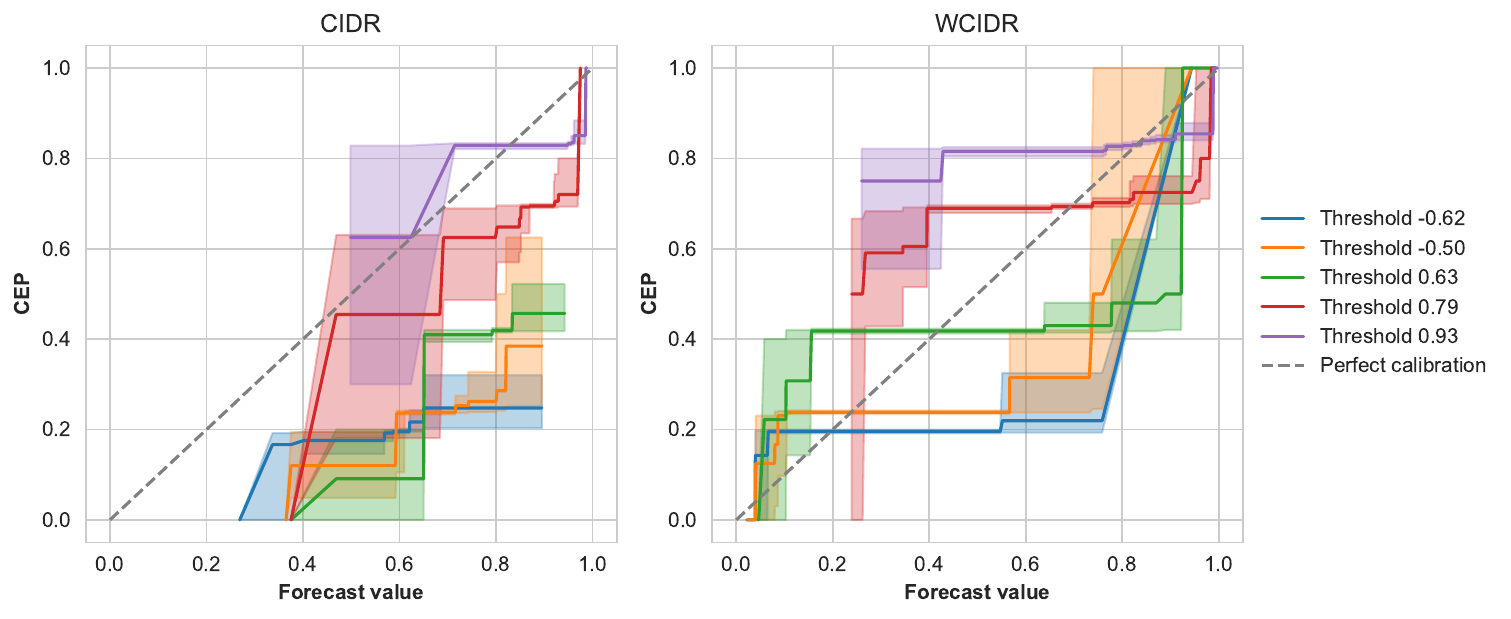}
    \caption{CIDR / WCIDR}
    \label{fig:aav_reldiag_lambda4_cidr}
  \end{subfigure}
  \caption{CORP reliability diagrams \citep{dimitriadis_stable_2021}. The x-axis shows predicted probabilities, and the y-axis shows the observed frequency of the outcome. A perfectly calibrated model lies along the diagonal: predictions match observed outcomes. Curves above the diagonal indicate underestimation, while curves below indicate overestimation. Threshold-wise CORP reliability diagrams for the AAV design task at inverse temperature $\lambda = 4$. Each panel shows weighted (right) and unweighted (left) calibration curves across several extreme packaging thresholds, comparing FCS-aware calibration (W*) with its unweighted counterpart. The shaded areas represent 95\% confidence intervals for the reliability diagram, constructed using bootstrapping.}
  \label{fig:aav_reldiag_lambda4}
\end{figure}

\begin{figure}
  \centering
  \begin{subfigure}[b]{0.8\textwidth}
    \centering
    \includegraphics[width=\linewidth]{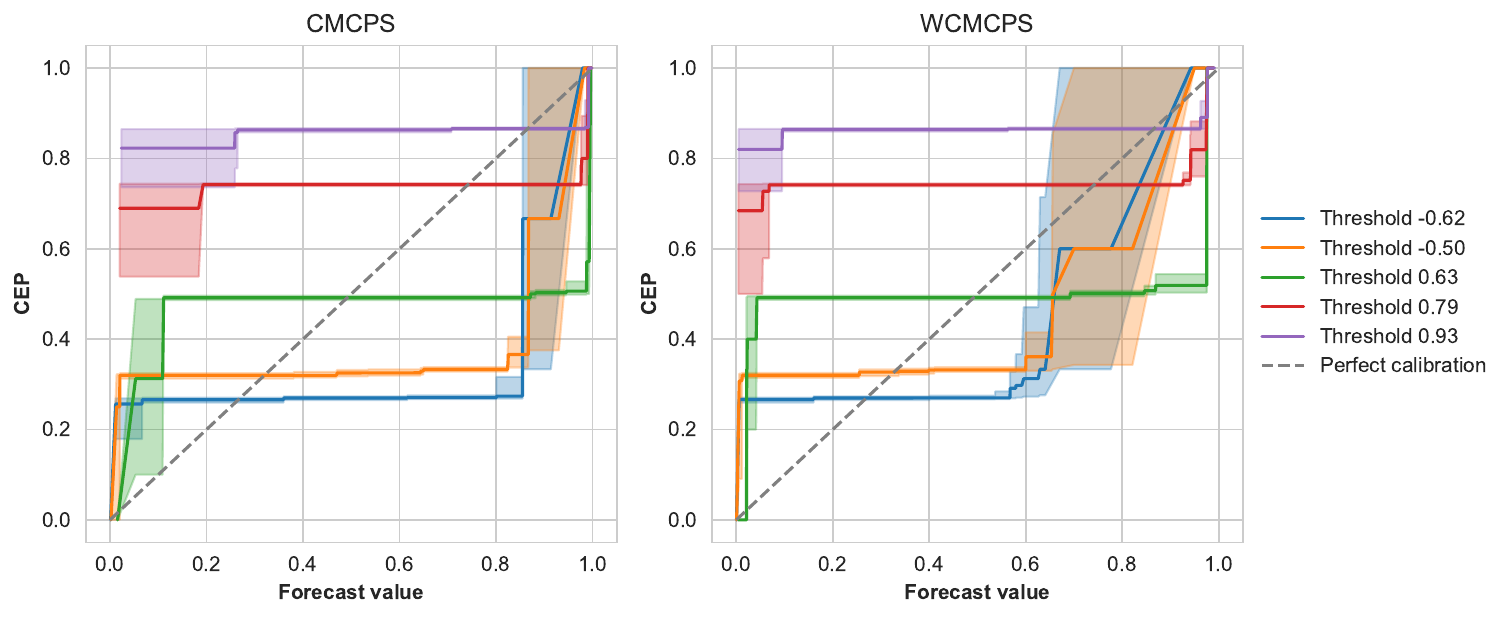}
    \caption{CMCPS / WCMCPS}
    \label{fig:aav_reldiag_lambda3_cmcps}
  \end{subfigure}\hfill
  \begin{subfigure}[b]{0.8\textwidth}
    \centering
    \includegraphics[width=\linewidth]{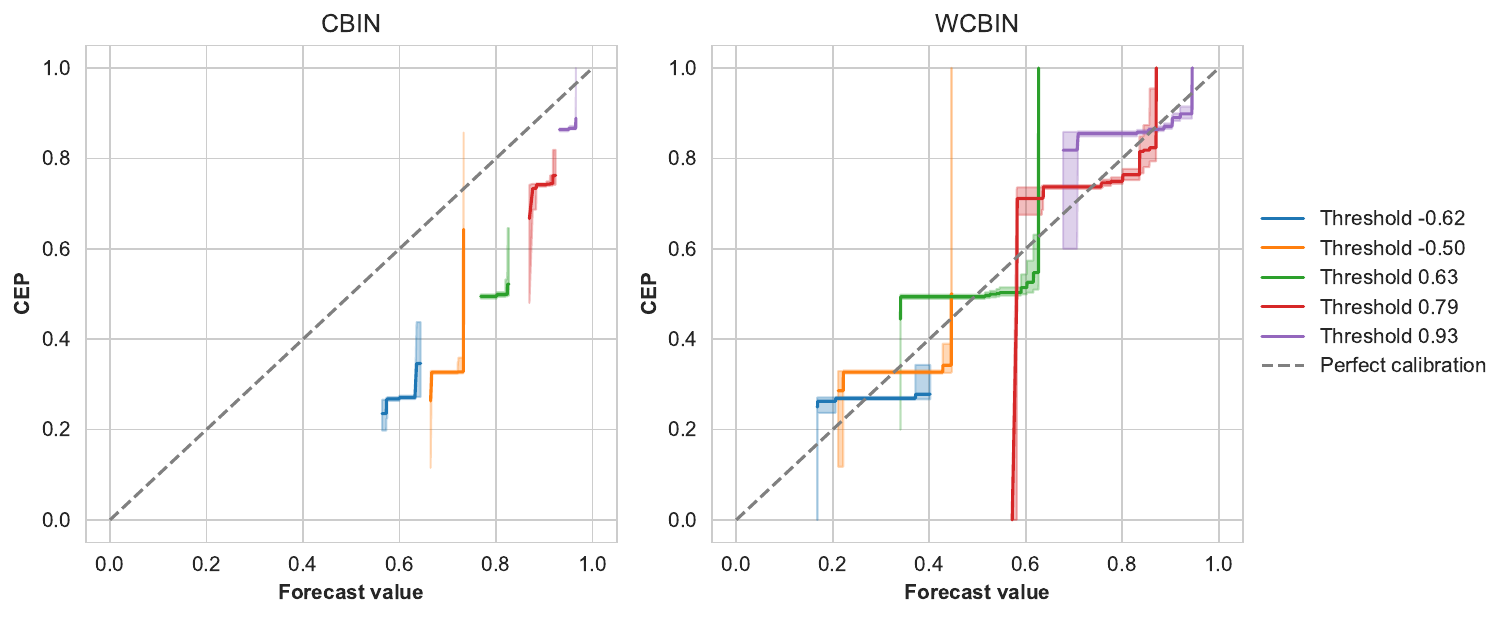}
    \caption{CBIN / WCBIN}
    \label{fig:aav_reldiag_lambda3_cbin}
  \end{subfigure}\hfill
  \begin{subfigure}[b]{0.8\textwidth}
    \centering
    \includegraphics[width=\linewidth]{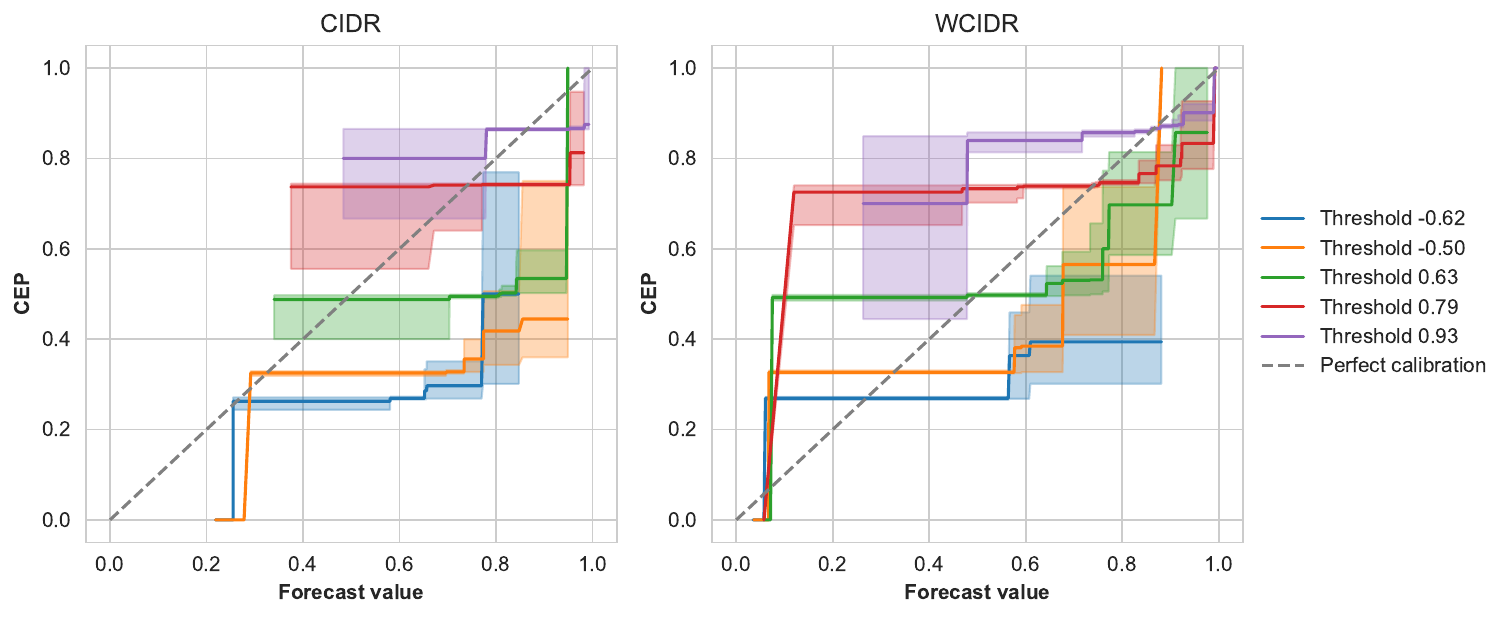}
    \caption{CIDR / WCIDR}
    \label{fig:aav_reldiag_lambda3_cidr}
  \end{subfigure}
  \caption{CORP reliability diagrams \citep{dimitriadis_stable_2021}. The x-axis shows predicted probabilities, and the y-axis shows the observed frequency of the outcome. A perfectly calibrated model lies along the diagonal: predictions match observed outcomes. Curves above the diagonal indicate underestimation, while curves below indicate overestimation. Threshold-wise CORP reliability diagrams for the AAV design task at inverse temperature $\lambda = 3$. Each panel shows weighted (right) and unweighted (left) calibration curves across several extreme packaging thresholds, comparing FCS-aware calibration (W*) with its unweighted counterpart. The shaded areas represent 95\% confidence intervals for the reliability diagram, constructed using bootstrapping.}
  \label{fig:aav_reldiag_lambda3}
\end{figure}

\begin{figure}
  \centering
  \begin{subfigure}[b]{0.8\textwidth}
    \centering
    \includegraphics[width=\linewidth]{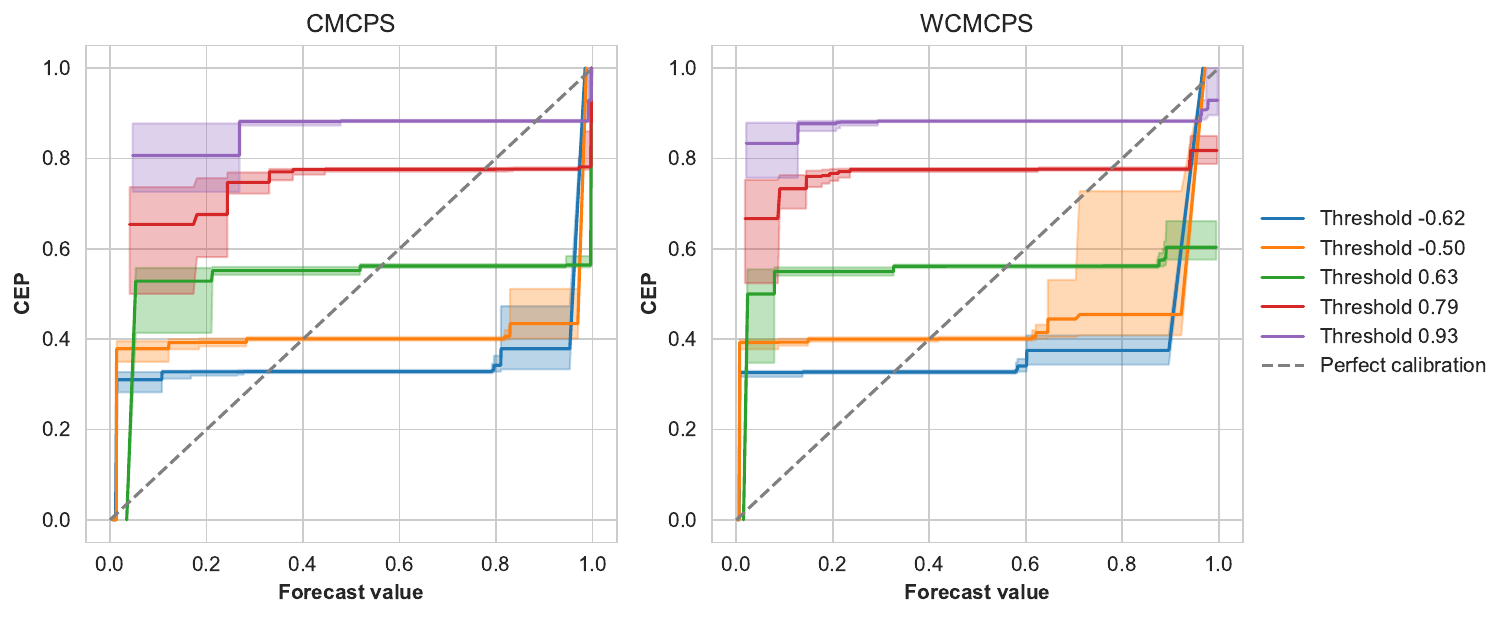}
    \caption{CMCPS / WCMCPS}
    \label{fig:aav_reldiag_lambda2_cmcps}
  \end{subfigure}\hfill
  \begin{subfigure}[b]{0.8\textwidth}
    \centering
    \includegraphics[width=\linewidth]{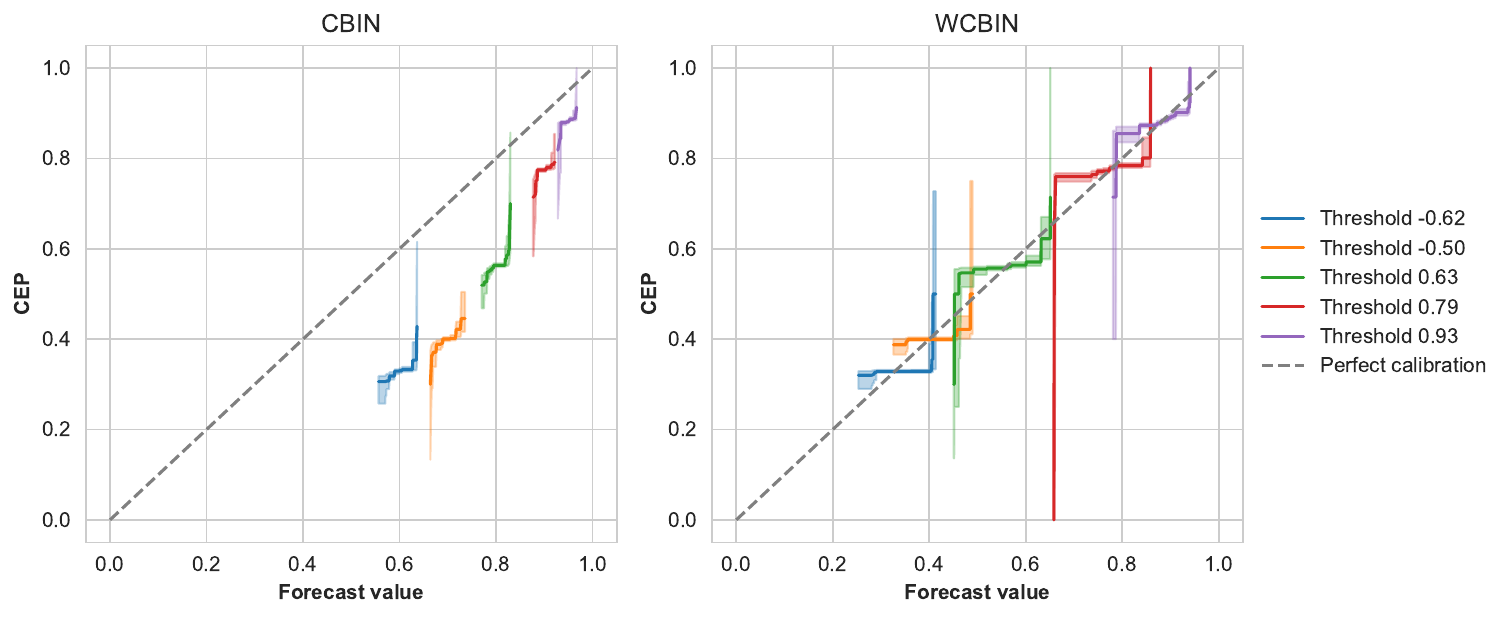}
    \caption{CBIN / WCBIN}
    \label{fig:aav_reldiag_lambda2_cbin}
  \end{subfigure}\hfill
  \begin{subfigure}[b]{0.8\textwidth}
    \centering
    \includegraphics[width=\linewidth]{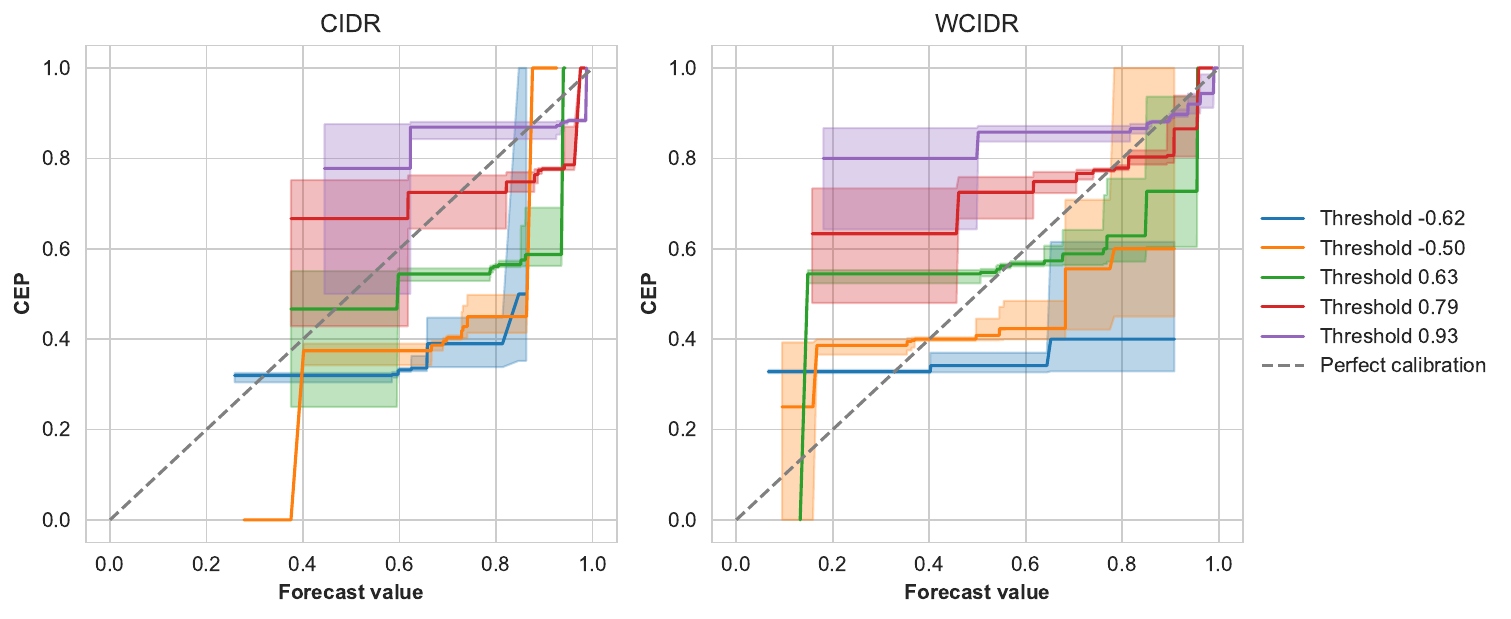}
    \caption{CIDR / WCIDR}
    \label{fig:aav_reldiag_lambda2_cidr}
  \end{subfigure}
  \caption{CORP reliability diagrams \citep{dimitriadis_stable_2021}. The x-axis shows predicted probabilities, and the y-axis shows the observed frequency of the outcome. A perfectly calibrated model lies along the diagonal: predictions match observed outcomes. Curves above the diagonal indicate underestimation, while curves below indicate overestimation. Threshold-wise CORP reliability diagrams for the AAV design task at inverse temperature $\lambda = 2$. Each panel shows weighted (right) and unweighted (left) calibration curves across several extreme packaging thresholds, comparing FCS-aware calibration (W*) with its unweighted counterpart. The shaded areas represent 95\% confidence intervals for the reliability diagram, constructed using bootstrapping.}
  \label{fig:aav_reldiag_lambda2}
\end{figure}

\begin{figure}
  \centering
  \begin{subfigure}[b]{0.8\textwidth}
    \centering
    \includegraphics[width=\linewidth]{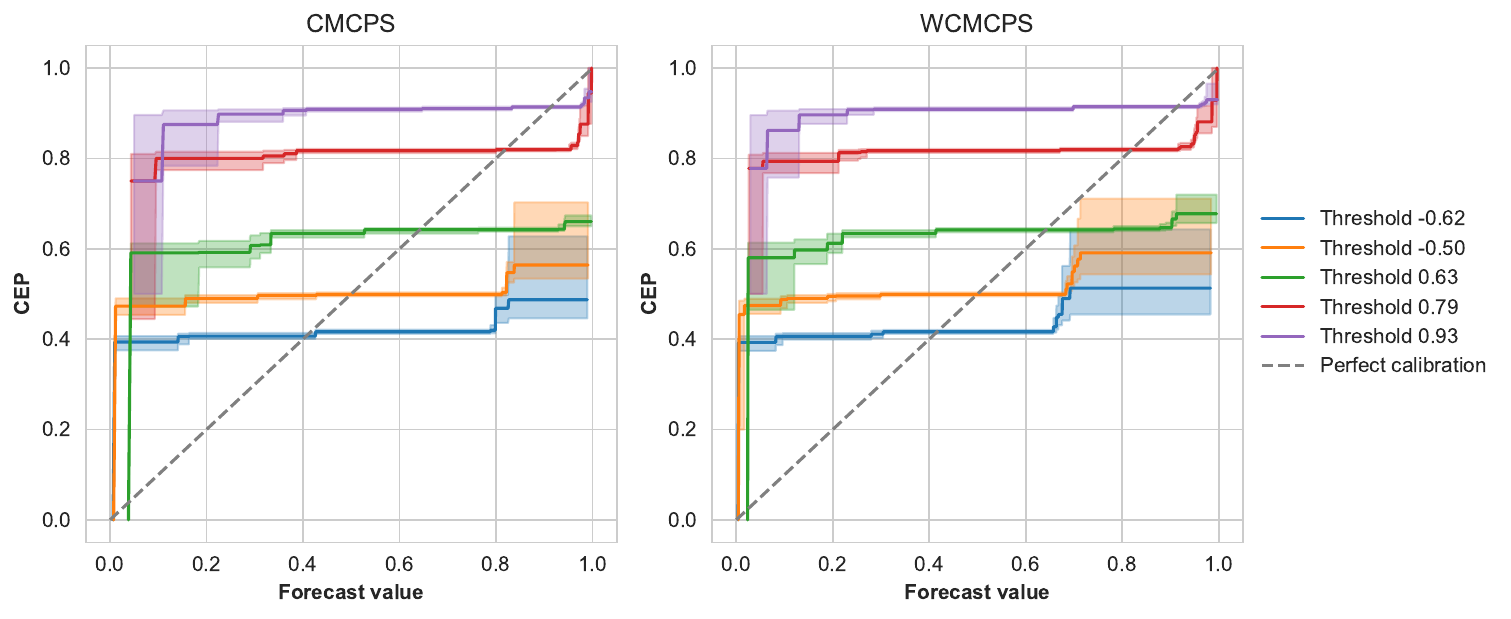}
    \caption{CMCPS / WCMCPS}
    \label{fig:aav_reldiag_lambda1_cmcps}
  \end{subfigure}\hfill
  \begin{subfigure}[b]{0.8\textwidth}
    \centering
    \includegraphics[width=\linewidth]{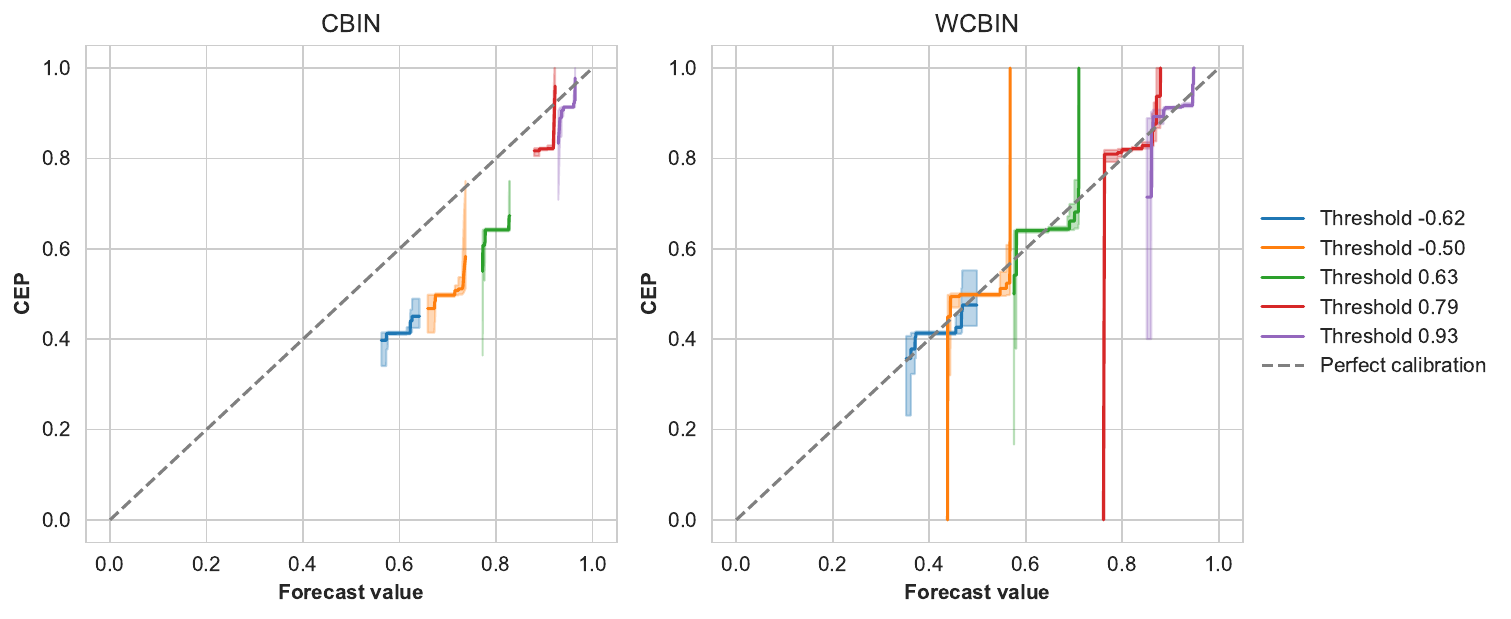}
    \caption{CBIN / WCBIN}
    \label{fig:aav_reldiag_lambda1_cbin}
  \end{subfigure}\hfill
  \begin{subfigure}[b]{0.8\textwidth}
    \centering
    \includegraphics[width=\linewidth]{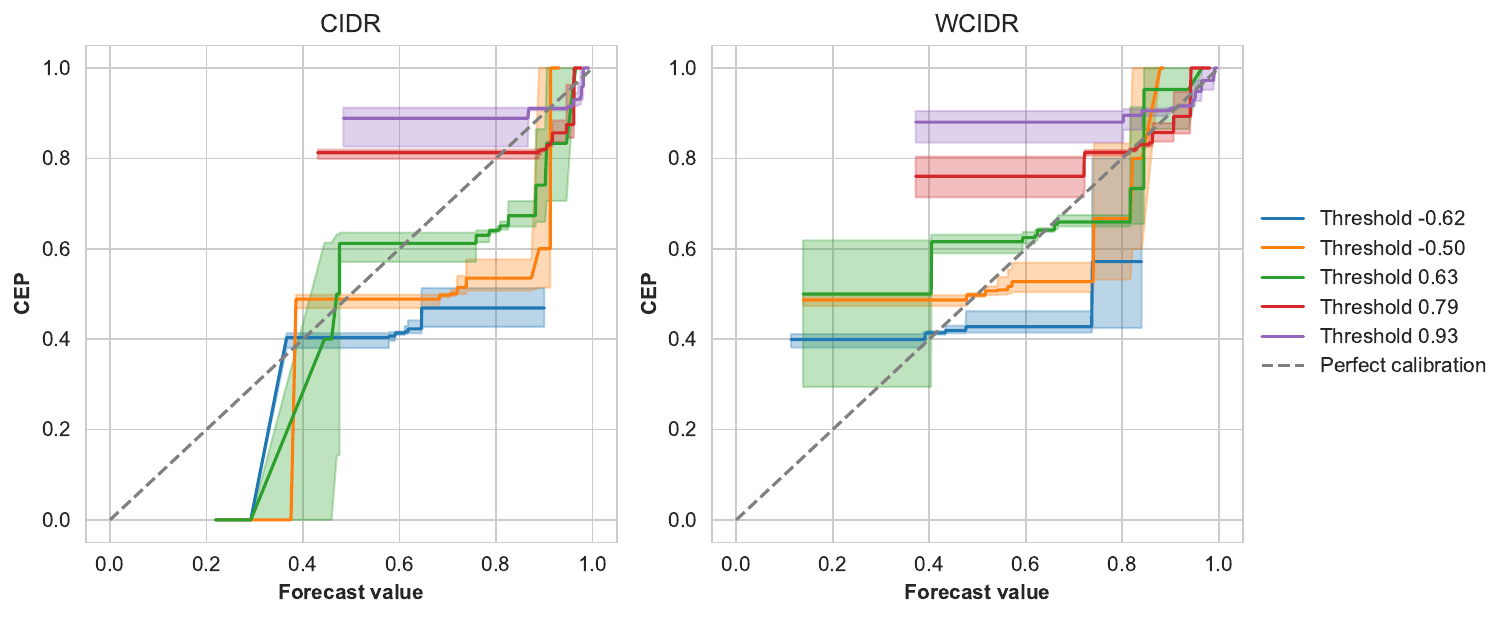}
    \caption{CIDR / WCIDR}
    \label{fig:aav_reldiag_lambda1_cidr}
  \end{subfigure}
  \caption{CORP reliability diagrams \citep{dimitriadis_stable_2021}. The x-axis shows predicted probabilities, and the y-axis shows the observed frequency of the outcome. A perfectly calibrated model lies along the diagonal: predictions match observed outcomes. Curves above the diagonal indicate underestimation, while curves below indicate overestimation. Threshold-wise CORP reliability diagrams for the AAV design task at inverse temperature $\lambda = 1$. Each panel shows weighted (right) and unweighted (left) calibration curves across several extreme packaging thresholds, comparing FCS-aware calibration (W*) with its unweighted counterpart. The shaded areas represent 95\% confidence intervals for the reliability diagram, constructed using bootstrapping.}
  \label{fig:aav_reldiag_lambda1}
\end{figure}

\clearpage
%%%%%%%%%%%%%%%%%%%%%%%%%%%%%%%%%%%%%%%%%%%%%%%%%%%%%%%%%%%%

% \newpage
% \input{checklist.tex}

\end{document}